\documentclass{article}

\usepackage{caption}
\usepackage{subcaption}
\usepackage{PRIMEarxiv}
\usepackage{natbib}
\usepackage[utf8]{inputenc} % allow utf-8 input
\usepackage[T1]{fontenc}    % use 8-bit T1 fonts
\usepackage{hyperref}       % hyperlinks
\usepackage{url}            % simple URL typesetting
\usepackage{booktabs}       % professional-quality tables
\usepackage{amsfonts}       % blackboard math symbols
\usepackage{nicefrac}       % compact symbols for 1/2, etc.
\usepackage{microtype}      % microtypography
\usepackage{lipsum}
\usepackage{fancyhdr}       % header
\usepackage{graphicx}       % graphics
\graphicspath{{media/}}     % organize your images and other figures under media/ folder

\usepackage{amsmath}
\usepackage{amssymb}
\usepackage{mathtools}
\usepackage{amsthm}

\theoremstyle{plain}
\newtheorem{theorem}{Theorem}[section]

\newtheorem{lemma}[theorem]{Lemma}

\theoremstyle{definition}

\theoremstyle{remark}

\title{APEX: Probing Neural Networks via Activation Perturbation
}

\author{
  REN Tao \\
  Aalborg University \\
  Denmark\\
  \texttt{taoren@es.aau.dk} \\
   \And
  Xiaoyu Luo \\
  Aalborg University \\
  Denmark\\
  \texttt{xilu@cs.aau.dk} \\
  \And
  Qiongxiu Li \\
  Aalborg University \\
  Denmark\\
  \texttt{qili@es.aau.dk}
}
\makeatletter
\renewenvironment{abstract}
{
  \centerline{\large \bfseries \scshape Abstract}
  \vspace{0.5em}
  \begin{list}{}{\setlength{\leftmargin}{2em}\setlength{\rightmargin}{2em}}
  \item\relax
}
{
  \end{list}
  \vspace{0.5em}
}
\makeatother

\usepackage{enumitem}
\setlist[itemize]{leftmargin=1.2em}

\begin{document}

\maketitle
\vspace{1em}

\begin{abstract}
Prior work on probing neural networks primarily relies on input-space analysis or parameter perturbation, both of which face fundamental limitations in accessing structural information encoded in intermediate representations. We introduce \textbf{A}ctivation \textbf{P}erturbation for \textbf{EX}ploration (APEX), an inference-time probing paradigm that perturbs hidden activations while keeping both inputs and model parameters fixed. We theoretically show that activation perturbation induces a principled transition from sample-dependent to model-dependent behavior by suppressing input-specific signals and amplifying representation-level structure, and further establish that input perturbation corresponds to a constrained special case of this framework. Through representative case studies, we demonstrate the practical advantages of APEX. In the small-noise regime, APEX provides a lightweight and efficient measure of sample regularity that aligns with established metrics, while also distinguishing structured from randomly labeled models and revealing semantically coherent prediction transitions. In the large-noise regime, APEX exposes training-induced model-level biases, including a pronounced concentration of predictions on the target class in backdoored models. Overall, our results show that APEX offers an effective perspective for exploring, and understanding neural networks beyond what is accessible from input space alone.
\end{abstract}

\begin{figure}[h]
\centering
\includegraphics[width=0.99\linewidth]{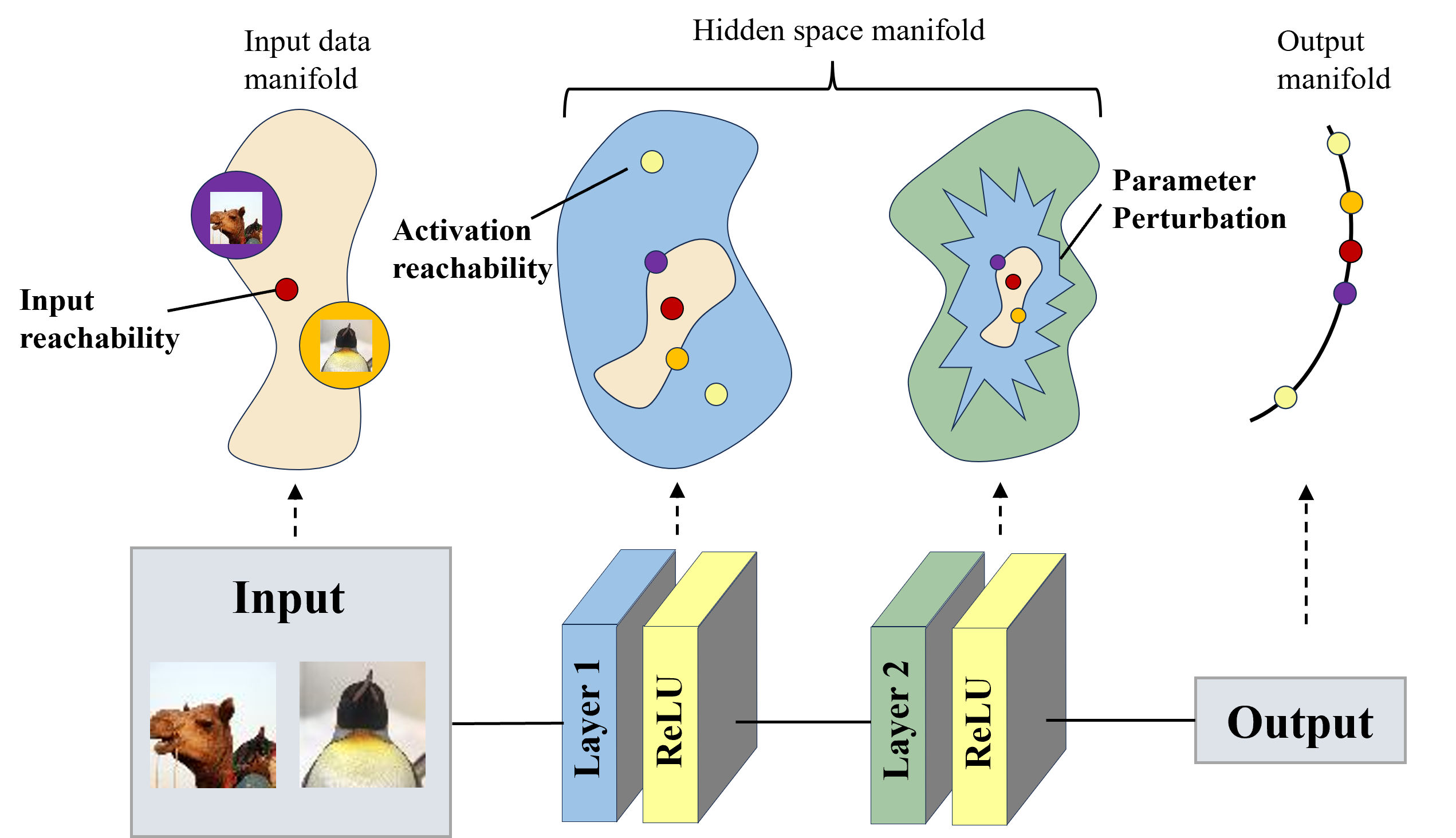}
\caption{
Conceptual illustration of probing mechanisms. Input perturbations are constrained by input-to-representation mapping, parameter perturbations modify the model itself, whereas \textbf{APEX} operates directly on hidden representations, enabling exploration beyond input reachability without modifying model parameters.}
\label{fig:hook}
\end{figure}

\section{Introduction}

Characterizing the information encoded in neural networks remains a central challenge in modern machine learning. Although training is driven by input--output supervision, a trained model ultimately operates through a hierarchy of hidden representations that mediate between data and prediction, encoding biases induced by the training data and regime~\cite{DeepRepresentationLearningonLong-TailedData,li2025trustworthymachinelearningmemorization,DataPoisoningAttacksintheTrainingPhase}.
As a result, phenomena such as generalization, memorization, robustness, and security vulnerabilities are fundamentally representation-level effects rather than purely input--output properties~\cite{li2018measuring, Internal-representation-tishby}. A central goal of model analysis is therefore to characterize how information is organized within this hidden representation space. 

A common approach to probing models is to analyze their responses to perturbations on inputs or parameters~\cite{perturbationbased}. Input perturbations have been used to study robustness and decision boundaries, while parameter perturbations are employed to assess model sensitivity~\cite{inputperturbation,decision_karimi, perturb-neuron-davis21a, Sharma2023InvestigatingWD}. However, input-space perturbations must propagate through successive nonlinear transformations and can only induce variations that lie within the image of the input-to-representation mapping~\cite{rankdiminishing}, rendering large portions of the internal representation space unreachable. Parameter perturbations alter the model itself, entangling the probing signal with changes to the object under study. Consequently, existing probing frameworks that operate through input or parameter perturbations may provide an incomplete view of the information encoded in internal representations.

To address this gap, we introduce Activation Perturbation for EXploration (APEX), an inference-time probing paradigm that perturbs hidden activations while keeping both inputs and parameters fixed. This design enables direct access to intermediate representations beyond the constraints imposed by input-space reachability. We show that input perturbation can be interpreted as a constrained special case of APEX, yielding a unified structural perspective on existing probing mechanisms. We further show that increasing activation noise induces a transition from sample-dependent to model-dependent behavior, revealing both local, instance-specific structure and global, model-level organization, which is analyzed in detail in Section~\ref{sec:activationperturbation}.  Figure~\ref{fig:hook} illustrates the conceptual differences among these probing paradigms.

% Motivated by this limitation, we propose to directly probe internal representations at intermediate layers. In this work, we introduce Activation Perturbation for EXploration (APEX), an inference-time probing paradigm that perturbs hidden activations while keeping inputs and parameters fixed, enabling direct exploration of internal representations beyond input-space reachability. 
% We show that input perturbation arises as a constrained special case of APEX, which admits a simple and principled structural interpretation.
% As formalized in Section~\ref{sec:activationperturbation}, increasing activation noise induces a transition from sample-dependent to model-dependent behavior, revealing both local sample-level structure and global model-level organization. Figure~\ref{fig:hook} provides an illustrative visualization to build intuition for this paradigm.

%Figure~\ref{fig:hook} provides a geometric illustration of the reachability limitation inherent to input-space probing: while input perturbations are confined to a low-dimensional input manifold, hidden representations progressively expand and reorganize across layers, forming a substantially larger and more expressive representation space, much of which is unreachable through input variations alone.
%Perturbing activations at intermediate layers therefore enables direct exploration of this internal space, bypassing input reachability constraints.

Consistent with this interpretation, APEX reveals two recurring empirical regimes that reflect how information is organized within learned representations.
In the small-noise regime, APEX provides a lightweight and efficient measure of sample regularity that aligns with established metrics, while also distinguishing models trained on structured versus randomly labeled data, and in controlled settings reveals prediction transitions that are consistent with semantic structure learned by the model. In the large-noise regime, predictions become input-independent and converge to a stationary, model-specific output distribution, exposing training-induced model-level biases, including a pronounced concentration of probability mass on the target class in backdoored models. Together, these observations indicate that APEX can expose representation-level structure in a more semantically aligned manner in our controlled experiments, whereas input- and parameter-level perturbations do not exhibit the same behavior under these settings. Our contributions are:
\begin{itemize}
    \item We provide evidence that perturbing hidden activations yields more semantically aligned probing behavior under controlled conditions, compared to input- or parameter-level perturbations.
    \item We demonstrate that activation perturbation reveals a unified two-regime behavior: under small noise, prediction stability is correlated with sample-level regularity, while under large noise, output distributions become input-agnostic and reflect model-level structure, including imprints of the training data.
    \item We show that the stationary output distribution induced by large activation noise encodes intrinsic model biases. In particular, backdoored models exhibit target-aligned concentration in this regime, whereas benign models produce more dispersed outputs, a distinction that is difficult to be revealed by input- or parameter-level perturbations.
\end{itemize}

\section{Related work}
\label{sec:relatedwork}

\paragraph{Sample Regularity and Memorization}
\label{sec:memandC}

Prior work has shown that not all training samples are treated equally by deep networks.
Long-tailed or atypical samples tend to be memorized \cite{li2025trustworthymachinelearningmemorization}, while representative samples generalize more robustly~\cite{feldman2020does}. 
Several metrics have been proposed to quantify this notion of \emph{sample regularity}, 
including memorization-based scores~\cite{whatneuralnetworkmem} and consistency-based measures~\cite{C-score}. 
Despite differences in formulation, these approaches consistently indicate that samples occupy decision regions of varying stability and geometric robustness. Additional discussion is provided in Appendix~\ref{app:longtailed}.

\paragraph{Backdoor Attacks}
\label{sec:BD}

Backdoor attacks introduce a small fraction of poisoned training samples that cause the trained model to associate a trigger pattern with a target class~\cite{ASurvey}. 
While the model behaves normally on clean inputs, the presence of the trigger induces a systematic and input-agnostic prediction shift.
Although backdoor mechanisms vary across attack designs~\cite{badnets,IAD,LIRA}, 
their common effect is to embed a persistent structural bias toward the target class within the model’s internal representations.

\paragraph{Perturbation-Based Probing.}
A common approach to probing neural networks is through input perturbations,
which relate robustness to local decision geometry.
For example, ~\citet{fawzi-geometry} analyze model responses to random and structured perturbations
constrained to low-dimensional subspaces, showing that the minimal perturbation norm
reflects the distance and curvature of the decision boundary. ~\citet{decision_karimi} further study local decision structure by generating perturbations near class boundaries and examining the behavior of borderline samples. Another line of work probes models via parameter perturbations. ~\citet{Sharma2023InvestigatingWD} quantify sensitivity by measuring expected changes
in activations and gradients under random weight perturbations,
and use this metric to analyze model stability and guide architecture design.

Recent work has shown that injecting noise into hidden activations at inference time can degrade safety guardrails in large language models~\cite{shahani2025noiseinjectionsystemicallydegrades}, indicating that activation-level perturbations can meaningfully affect model behavior. This observation highlights the potential of activation perturbation as a probe of internal model structure.

% \section{Proposed Activation Perturbation}
% \label{sec:activationperturbation}

% Consider an $L$-layer neural network 
% $f_\theta:\mathbb{R}^d \to \mathbb{R}^c$ 

% \[
% z_\ell = W_\ell a_{\ell-1} + b_\ell, \qquad 
% a_\ell = \phi(z_\ell),
% \]

% where $W_\ell$, $b_\ell$, and $\phi(\cdot)$ denote the layer's weight matrix, bias, and activation function respectively.

% To probe how signals propagate through the nonlinear layers of a trained model, 
% we perform repeated forward passes while injecting Gaussian noise after each activation:
% \[
% \tilde{a}_\ell 
% = \phi(z_\ell) + \sigma\,\xi_\ell, 
% \qquad
% \xi_\ell \sim \mathcal{N}(0,I).
% \]
% where \( \xi_\ell \sim \mathcal{N}(0, I) \) are Gaussian noise vector with i.i.d. entries, and \( \sigma \) is a positive noise strength scaling factor. The input $x$, the parameters $\theta$, and the pre-activations $z_\ell$ remain unchanged across runs. All noises are independent across layers and repetitions.

\section{Activation Perturbation Framework}
\label{sec:activationperturbation}

We now proceed to explain details of the proposed \textbf{APEX}, a simple yet effective framework for analyzing neural network behavior by introducing controlled stochastic perturbations into intermediate activations and observing the resulting output variations. The core idea of APEX is to probe the internal sensitivity and response patterns of a trained model without modifying its parameters or input data, thereby enabling direct exploration of internal representation space beyond input reachability. Despite its conceptual simplicity, this approach enables systematic characterization of prediction stability and reveals abnormal behaviors that are difficult to capture through standard approaches such as input and parameter perturbations.

Formally, consider an $L$-layer neural network 
$f_\theta:\mathbb{R}^d \to \mathbb{R}^c$ with layer-wise transformations
\[
z_\ell = W_\ell a_{\ell-1} + b_\ell, \qquad 
a_\ell = \phi(z_\ell),
\]
where $W_\ell$ and $b_\ell$ are the weight matrix and bias term of layer $\ell$, 
and $\phi(\cdot)$ is the activation function, so that $z_\ell$ and $a_\ell=\phi(z_\ell)$
denote the pre- and post-activation representations, respectively. To probe the structure and sensitivity of internal representations in a trained model, we perform repeated forward passes while injecting Gaussian noise after each activation:
\[
\tilde{a}_\ell 
= \phi(z_\ell) + \sigma\,\xi_\ell, 
\qquad
\xi_\ell \sim \mathcal{N}(0,I).
\]
Here, $\xi_\ell$ denotes a Gaussian noise vector with i.i.d.\ entries, and $\sigma$ is a positive scaling factor controlling the perturbation magnitude. The input $x$, model parameters $\theta$, and pre-activations $z_\ell$ remain fixed across runs. Noise realizations are independently drawn across layers and repetitions.

\subsection{Stochastic Output Distribution Estimation}

For a fixed input $x$ and class set $\mathcal{Y}$, we characterize the model’s predictive behavior under APEX by repeatedly performing stochastic forward passes with injected noise. Specifically, we conduct $T$ independent perturbation trials and record the top-1 predicted class at each run. Let $k_t^{\ast}$ denote the predicted class at trial $t$. We estimate the empirical probability of predicting class $k$ under perturbation magnitude $\sigma$ as
\[
\hat{P}_x(k;\sigma)
= \frac{1}{T}\sum_{t=1}^{T}
\mathbf{1}\!\left(k_t^{\ast}=k\right),
\qquad k\in\mathcal{Y}.
\]
This sampling-based estimation procedure approximates the output distribution induced by stochastic perturbation.

% \subsection{Monte Carlo Simulation.} For a fixed input $x$ and a set of classes $\mathcal{Y}$, we characterize the model’s perturbed behavior 
% by running $T$ noisy forward passes and recording the top–1 prediction 
% on each run.  
% Let $k_t^{\ast}$ denote the predicted class at run $t$.  
% The empirical probability of predicting class $k$ under perturbation~$\sigma$ is
% \[
% \hat{P}_x(k;\sigma)
% = \frac{1}{T}\sum_{t=1}^{T} 
% \mathbf{1}\!\left(k_t^{\ast}=k\right),
% \qquad k\in\mathcal{Y}.
% \]
% These empirical output distributions summarize how often each class is selected across Monte Carlo runs. 

\subsection{Theoretical Characterization of APEX}
\label{sec:mathInterpretation}

Activation perturbation does not directly modify the input or the model parameters, yet it reveals both sample-level regularity and model-level structure. Mathematically, it gives rise to a decomposition of the forward signal. 

\begin{theorem}[Decomposition of activation perturbation]
\label{thm:activation-decomposition}
Assume inputs lie in a bounded set $\mathcal{X}=\{x:\|x\|\le R\}$ under an induced norm $\|\cdot\|$.
Then for each layer $\ell\in\{1,\dots,L\}$ there exist a noise-dependent vector $v_\ell$,
independent of $x$, and a residual $r_\ell(x;\sigma)$ such that for all $x\in\mathcal{X}$ and $\sigma>0$,
\begin{equation}
\label{eq:activation-decomp}
\tilde{a}_\ell(x;\sigma)=\sigma v_\ell + r_\ell(x;\sigma),
\qquad
\|r_\ell(x;\sigma)\|\le B_\ell(R;W,b),
\end{equation}
where $B_\ell(R;W,b)$ is independent of $\sigma$.
A constructive definition of $(v_\ell,r_\ell)$ and the bound are provided in Appendix~\ref{app:DecompositionProP}
and Appendix~\ref{app:boundness}.
\end{theorem}

Applying the decomposition to the final layer yields a corresponding separation at the logits. For $\sigma>0$, Let $\tilde{a}_L(x;\sigma)$ denote the final hidden activation and
$s(x;\sigma)=U\tilde{a}_L(x;\sigma)+c$ the logits. Using
$\tilde{a}_L(x;\sigma)=\sigma v_L+r_L(x;\sigma)$, we obtain
\[
s(x;\sigma)=\sigma\,U v_L + \underbrace{(U r_L(x;\sigma)+c)}_{=:e(x;\sigma)}.
\]
Thus the predicted class $k^\ast$ satisfies
\[
\arg\max_k s_k(x;\sigma)
=\arg\max_k\!\left((Uv_L)_k + \frac{e_k(x;\sigma)}{\sigma}\right).
\]
Since $\|e(x;\sigma)\|_\infty \le \|U\|_\infty B_L(R;W,b)+\|c\|_\infty$ uniformly over
$x\in\mathcal{X}$ and $\sigma>0$ (Appendix~\ref{app:boundness}),
the input-dependent term $e(x;\sigma)/\sigma$ vanishes as $\sigma$ grows.
Consequently, the prediction becomes increasingly governed (in distribution) by the
model-induced random logits $Uv_L$ rather than the specific input $x$.
More precisely, conditioning on a fixed noise draw, once $\sigma$ exceeds a margin-based threshold,
the predicted label becomes independent of $x$; under noise resampling at inference time,
this implies that the prediction distribution converges to a model-specific limit
(Appendix~\ref{app:GPD}).

When $\sigma = 0$, the network reduces to its standard deterministic form and the output is entirely
determined by the unperturbed activations. For $\sigma > 0$, the decomposition induces two regimes:

\textbf{Small-noise regime:} When $\sigma$ is small, the residual term dominates, and the perturbed
network behaves similarly to the original deterministic model; its output remains sensitive to the
input sample.

\textbf{Large-noise regime:} When $\sigma$ is large, the $\sigma v_\ell$ component dominates and the
prediction becomes increasingly insensitive to the input (in distribution). Appendix~\ref{app:GPD}
formalizes the corresponding stationary, model-specific prediction distribution.

\paragraph{Input perturbation as a special case.}
Input samples play a central role in shaping the geometry and decision boundaries of neural networks,
reflecting a close connection between data distributions and learned representations
\cite{decision_karimi, fawzi-geometry, inputperturbation}. From a functional perspective,
input perturbation can be viewed as a special case of activation perturbation.
Perturbing the input $x$ by $\varepsilon$ induces a corresponding perturbation at any intermediate layer $\ell$,
\[
a_\ell(x+\varepsilon)=a_\ell(x)+\Delta_\ell(x,\varepsilon),
\]
where $
\Delta_\ell(x,\varepsilon):=a_\ell(x+\varepsilon)-a_\ell(x),$ is entirely determined by the prefix network up to layer $\ell$.
Thus, input perturbation amounts to injecting a \emph{constrained} perturbation into the activation space,
followed by the same suffix network. Taking $\ell=0$ recovers the input layer itself, making input noise a
degenerate instance of activation noise.

For small $\varepsilon$, a first-order expansion yields
$\Delta_\ell(x,\varepsilon)\approx J_{a_\ell}(x)\varepsilon$,
where $J_{a_\ell}(x)$ denotes the Jacobian of the prefix network \cite{novak2018sensitivity, LIUjacobian}.
Consequently, input perturbations can only induce activation changes lying in the image of this Jacobian,
which typically forms a highly constrained subset of the representation space \cite{rankdiminishing}.

Together, this analysis explains why APEX separates sample-dependent variability from model-level structure,
and why probing at the activation level enables access to representation properties that are fundamentally
inaccessible from the input space.

\section{Representation Structure under Small Activation Noise}
\label{sec:phe1}

\begin{figure}[h!]  % 'h!'表示尽可能在当前位置插入
\centering  % 图片居中
\includegraphics[width=\linewidth]{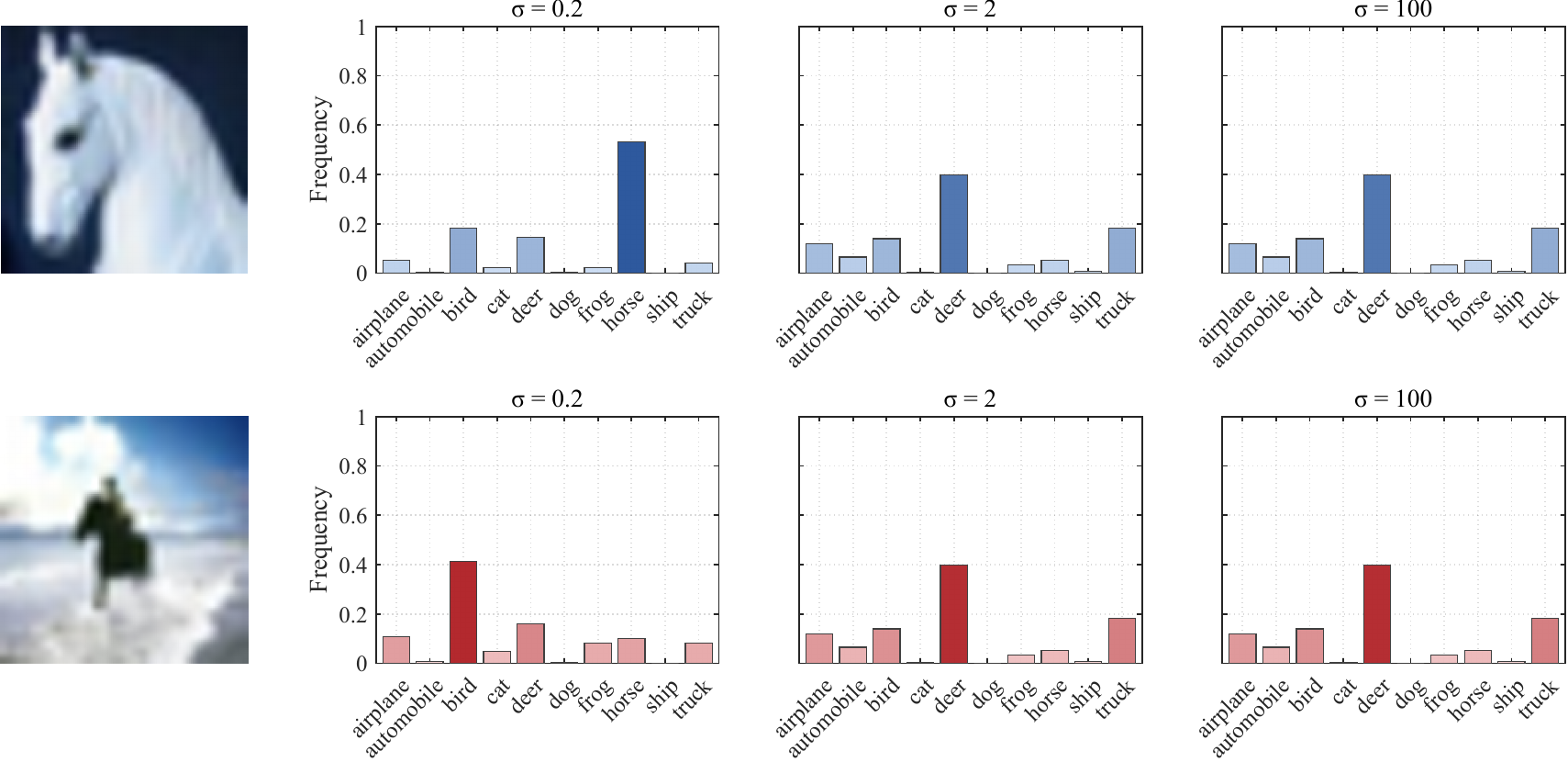}  

\caption{Real examples of the output distributions obtained from two ``horse'' samples.}  % 图片标题
\label{fig:OD-example}  % 图片标签，用于交叉引用
\end{figure}

In this section, we study the behavior of APEX in the small-noise regime.
Our goal is to examine whether, when sample-dependent structure is preserved,
APEX probes internal representations in a manner that is consistent with
the semantic structure learned by the model in controlled settings, and whether this behavior differs from that observed under input- or parameter-level perturbations.
 Detailed experimental setup for following sections is provided in
Appendix~\ref{app:trainstats}.

We begin by illustrating this behavior through a concrete example. Figure~\ref{fig:OD-example} shows the output distributions obtained from repeated forward passes
of a trained ResNet-18 model on two CIFAR-10 samples under small activation noise. For a typical ``horse'' image, the predicted class remains dominant across Monte Carlo runs even at $\sigma=0.2$.
In contrast, an atypical ``horse'' exhibits early prediction shifts toward semantically related classes such as ``bird'' or ``airplane''. These transitions are not arbitrary, but follow semantic proximity in the learned representation space.

\begin{figure}[h!]  % 'h!'表示尽可能在当前位置插入
\centering  % 图片居中
\includegraphics[width=0.7\linewidth]{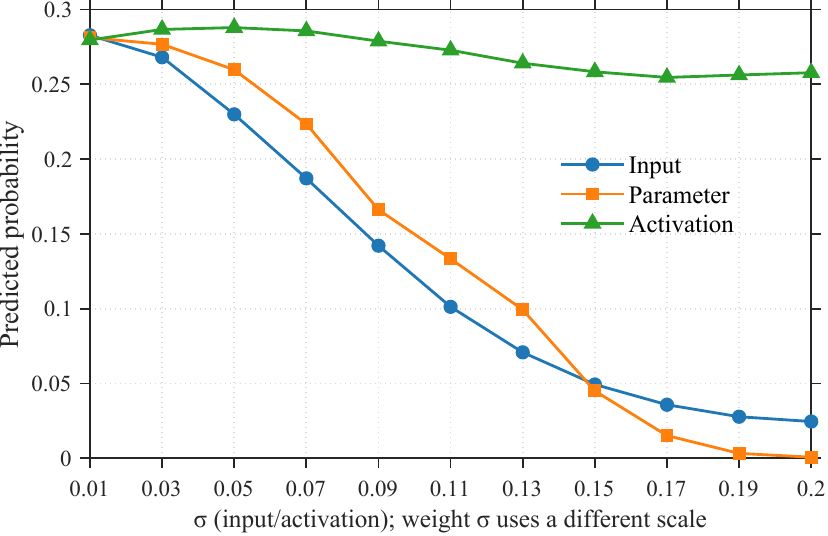}  
\caption{Semantic alignment under small noise.
Predicted probability of the reassigned class in a controlled CIFAR-10 setup
where two classes share the same input distribution.
Only activation perturbation induces a monotonic transfer between the two classes,
while input and parameter perturbations do not. (The noise magnitude for weight perturbation is rescaled, and the corresponding $\sigma$ values equal the x-axis values multiplied by $10^{-1}$)
}
\label{fig:split}
\end{figure}

\subsection{Semantic Alignment of Output Distributions}

To directly probe representation-level structure under small activation noise,
we design a controlled experiment in which class labels are decoupled from input semantics. Specifically, we construct multiple CIFAR-10 models where two randomly chosen classes
share the same underlying input distribution:
all original samples from one class are removed,
and half of the samples from the other class are randomly reassigned to it.
Under this setup, the two classes are semantically indistinguishable,
and any preference for one over the other must arise from how the model organizes
its internal representations.

In this scenario, a perturbation that faithfully explores the representation space
should increasingly transfer samples between the two classes as noise grows.
Figure~\ref{fig:split} reports the predicted probability of the reassigned class
for samples originating from the source class under increasing perturbation strength.
While input and parameter perturbations (implementation details in Appendix~\ref{app:perturb_impl}) do not exhibit a consistent shift,
activation perturbation induces a clear and consistent preference toward the reassigned class. This behavior indicates that, in this controlled setting, activation perturbation induces prediction transitions that are more closely aligned with the representation structure learned by the model, while input- and parameter-level perturbations do not exhibit comparable transitions under the same conditions.

\subsection{Revealing Training Imprints}
\label{subsec:escape_fragmentation}

\begin{figure}[htb]
\centering
\includegraphics[width=0.7\linewidth]{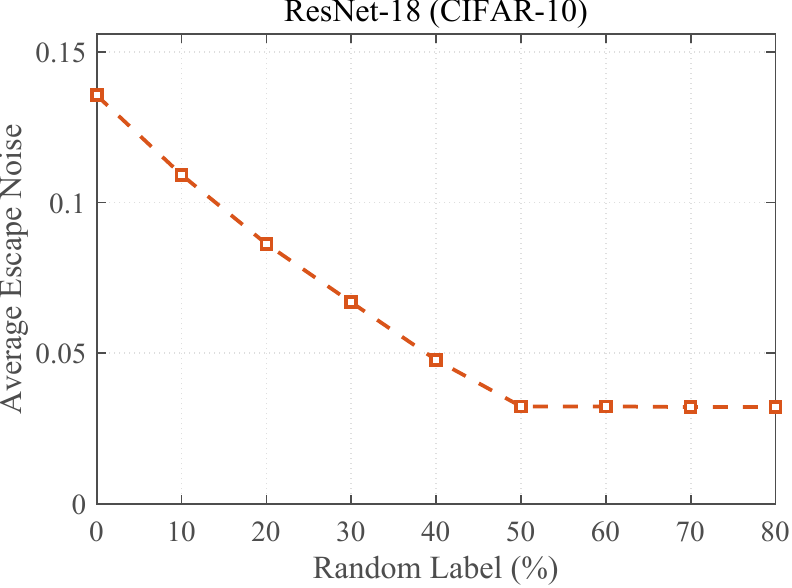}
\caption{Average escape noise for models with different percentages of samples randomly labeled.}
\label{fig:randomlabel}
\end{figure}

Beyond semantic alignment, we next examine whether APEX can reveal how training distributions shape local representation structure.
To this end, we consider models trained on datasets with increasing proportions
of randomly assigned labels, a controlled setting known to induce fragmented and irregular decision regions. We examine whether APEX faithfully probes this training-induced fragmentation. 

To quantify this effect, we introduce \emph{escape noise} as a scalar descriptor
to characterize when a sample’s prediction becomes unstable. Formally, we define the escape noise of a sample as the smallest noise magnitude $\sigma$
at which the predicted probability of its predicted class $k^\ast$ drops below a fixed reference level:
\begin{align}
  \hat P_x(k^\ast;\sigma) \le 0.5 .
\end{align}
This operational definition provides a convenient way to summarize the relative sensitivity of samples
to activation perturbations.

Figure~\ref{fig:randomlabel} reports results on CIFAR-10 with a ResNet-18 model, showing that the average escape noise decreases monotonically with the random-label ratio. This trend indicates that samples escape their original predictions under much weaker perturbations when local class structure is disrupted. Consistent behavior is observed across additional dataset--architecture pairs, as reported in Appendix~\ref{app:randomlabel}. This behavior indicates that APEX in the small-noise regime
is sensitive to training-induced fragmentation of local representation structure.

By contrast, when the same experiment is conducted with input- or parameter-level perturbations,
we do not observe a consistent relationship between perturbation strength and random-label ratio
(Appendix~\ref{app:randomlabel_ablation}), suggesting that these perturbation sites are less effective at
revealing training-induced differences in representation behavior under this setting.

\subsection{Connection to Sample-Level Regularity}

In the small-noise regime, activation perturbation naturally induces
sample-dependent prediction stability, reflecting how individual samples
are situated within the learned representation space.
This mirrors prior observations that prediction instability under input-level
Gaussian perturbations correlates with a sample’s proximity to the decision boundary
\cite{inputperturbation,ford2019adversarial}.
Having established that APEX provides a more general probe of
representation structure, we briefly examine how its induced stability
relates to established notions of sample-level regularity.

\begin{figure}[h!]
    \centering
    % ---- Subfigure 1 ----
    \begin{subfigure}[b]{0.5\textwidth}
        \centering
        \includegraphics[width=\textwidth]{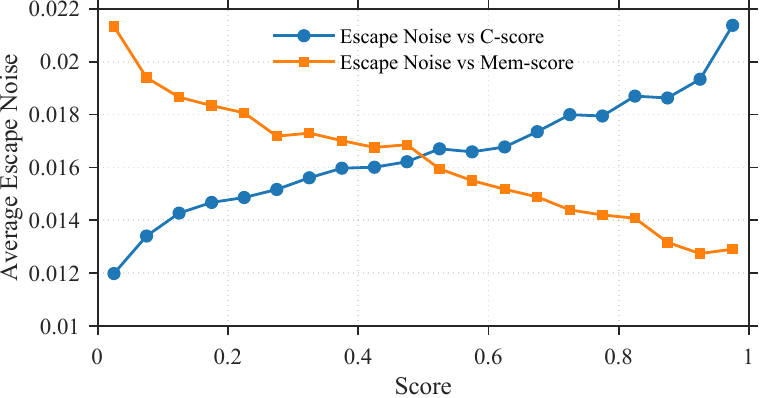}
        \caption{ImageNet.}
        \label{fig:imagenet_EN}
    \end{subfigure}
    % ---- Subfigure 2 ----
    \begin{subfigure}[b]{0.5\textwidth}
        \centering
        \includegraphics[width=\textwidth]{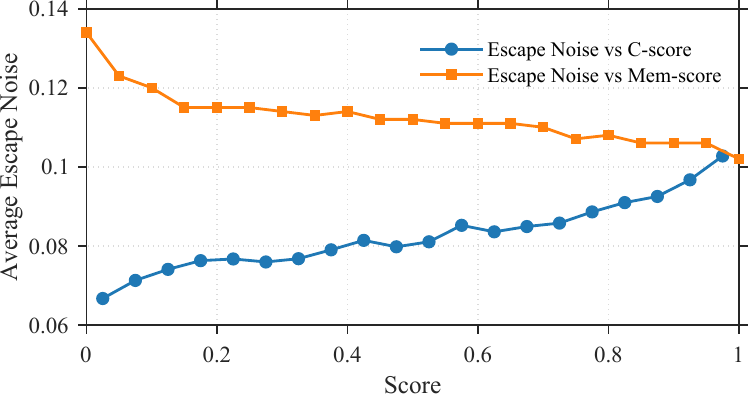}
        \caption{CIFAR-100.}
        \label{fig:imagenet_Acc}
    \end{subfigure}

    \caption{
    Relationship between average escape noise and memorization score (Mem-score), consistency score (C-score) on ImageNet and CIFAR-100.  
    }
    \label{fig:EN_regularity}
\end{figure}

\textbf{Alignment with Established Regularity Measures.}

We compare escape noise with established sample regularity metrics, including memorization score~\cite{feldman2020does,whatneuralnetworkmem} and consistency score~\cite{C-score}. As shown in Figure~\ref{fig:EN_regularity}, escape noise exhibits strong correlations with both metrics on ImageNet and CIFAR-100 for ResNet-50 models (C-scores for CIFAR-100 follow the Inception-based setting of~\cite{C-score}), indicating that prediction stability under small activation noise
is broadly consistent with known notions of sample regularity.

Samples with high C-scores and low memorization scores consistently exhibit large escape-noise values, whereas irregular or highly memorized samples escape under substantially smaller perturbations. Similar trends are also observed under input-level noise (Appendix~\ref{app:EN_regularity_detail}), confirming that activation perturbation preserves known stability-regularity relationships. 
Examples are visualized in Figure~\ref{fig:MemAlign}. 

\begin{figure}[h!]  % 'h!'表示尽可能在当前位置插入
\centering  % 图片居中
\includegraphics[width=\linewidth]{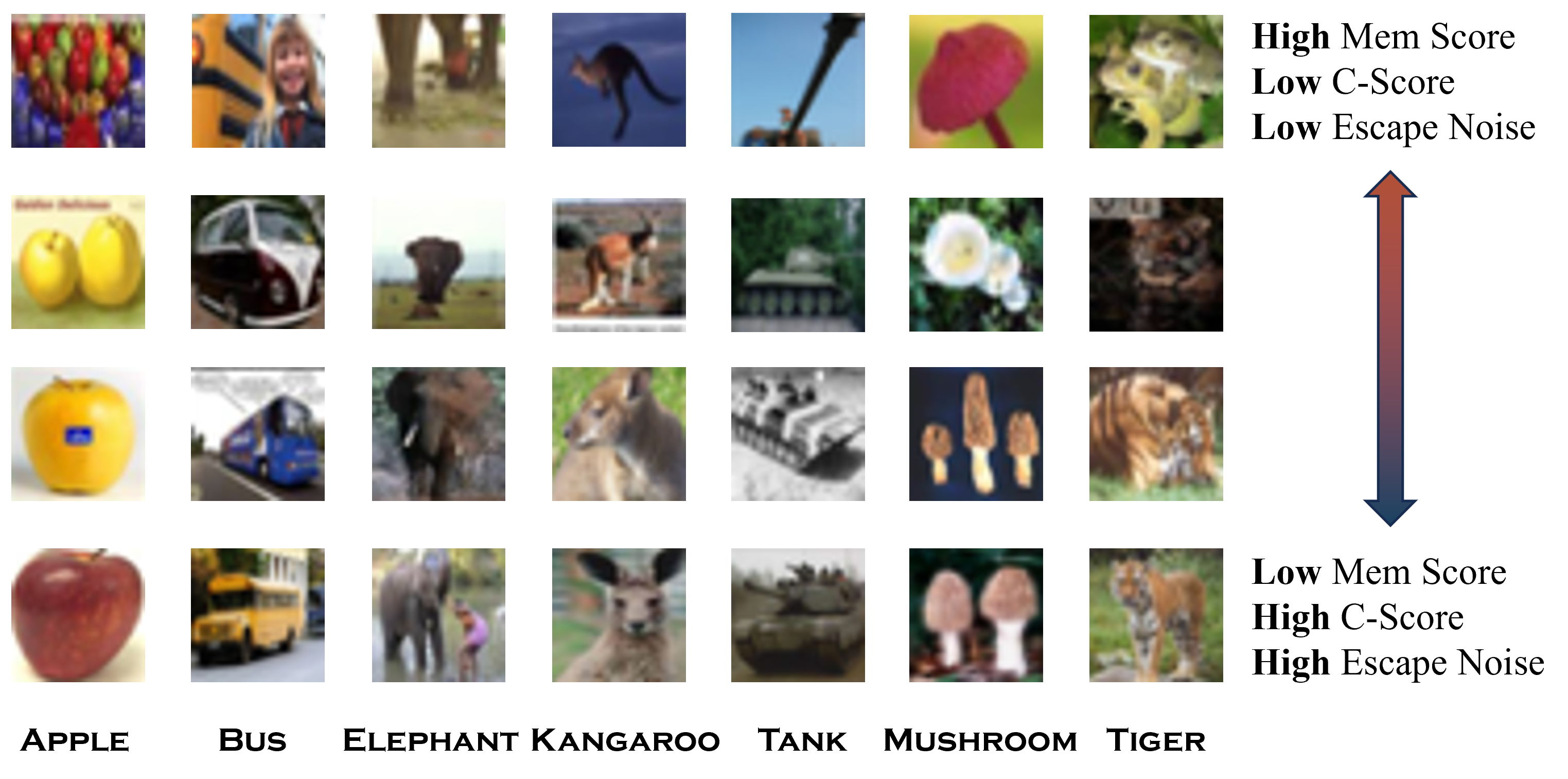}  
\caption{Examples of the regular and irregular samples and their ranges of memorization scores, consistency scores and escape noises.}
\label{fig:MemAlign}
\end{figure}

\textbf{Computational Efficiency}

The above results suggest that escape noise provides a simple and effective
scalar summary of this stability, and can therefore serve as a lightweight metric for quantifying sample-level regularity. Unlike memorization-based metrics \cite{feldman2020does,whatneuralnetworkmem}, which require training large ensembles
of shadow models, escape noise can be computed using a single trained model with lightweight Monte Carlo inference.
This makes it a scalable alternative for large-scale analysis, while remaining consistent with established regularity indicators.

\section{Model-Level Stationary Behavior under Large Activation Noise}
\label{sec:phe2}

\subsection{Implications of the Large-Noise Regime}

The theoretical analysis in Section~\ref{sec:mathInterpretation} predicts that,
as the perturbation magnitude becomes sufficiently large,
the noise-driven component dominates the forward signal,
causing the network output to become independent of the input sample.
This behavior is already hinted at in Figure~\ref{fig:OD-example},
where the rightmost columns show that, under large noise,
the output distributions of different samples become identical.

We verify this prediction quantitatively using a controlled experiment.
For a fixed noise magnitude $\sigma$, we collect the output distributions
obtained from multiple randomly selected input samples
and compute the average pairwise JS divergence between them. In addition, for each sample, we measure the JS divergence between its output distributions at consecutive noise levels. As the noise magnitude increases, both quantities rapidly decrease and
converge to near-zero values, indicating that
(i) output distributions produced from different input samples
become identical, and
(ii) the distributions stabilize across increasing noise strengths. The results are shown in Appendix~\ref{app:stationary}. 

Despite this shared sample-independence, the stationary output distributions induced by different perturbation sites exhibit clear qualitative differences. Under input-level perturbations, the stationary output distribution tends to collapse onto a dominant class,
suggesting a degenerate form of stationarity.
In contrast, activation perturbation yields stationary output distributions
that remain substantially more dispersed across classes,
indicating a richer and more stable model-level signature. A quantitative comparison based on output entropy
is reported in Figure~\ref{fig:entropy_stationary}.

\begin{figure}[h]
    \centering
    \includegraphics[width=0.6\linewidth]{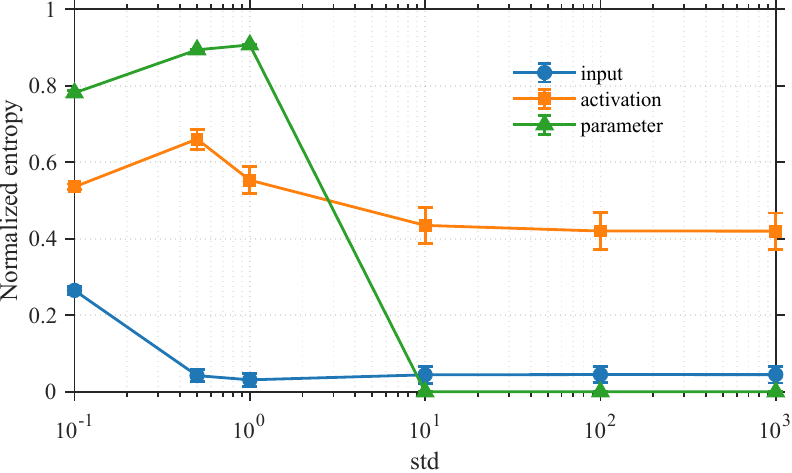}
    \caption{
    Normalized entropy of output distributions
    under different $\sigma$. 
    Results are averaged over ten Inception models on CIFAR-100.
    }
    \label{fig:entropy_stationary}
\end{figure}

\begin{figure*}[h!]  % 'h!'表示尽可能在当前位置插入
\centering  % 图片居中
\includegraphics[width=\linewidth]{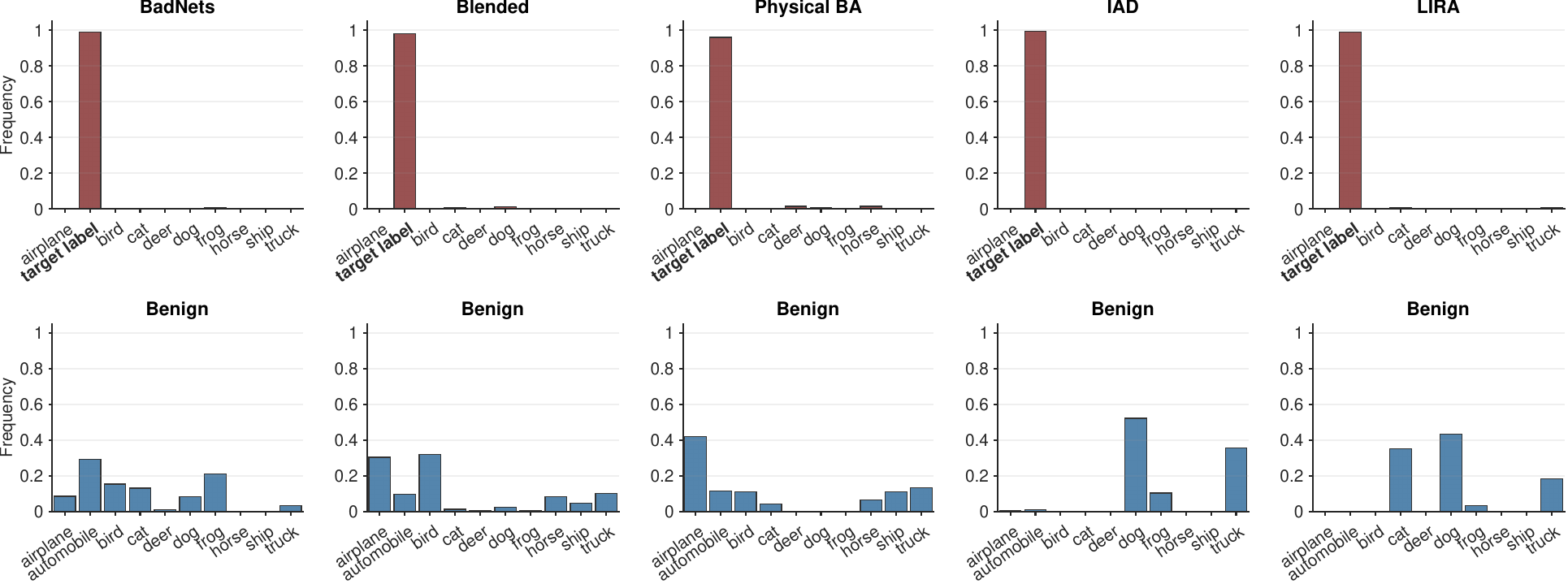}  % 图片宽度占满双栏
\caption{Real examples of model-level distributions for benign and backdoored models under different backdoor attacks: BadNets~\cite{badnets}, Blended~\cite{Blended}, Physical BA~\cite{physical}, IAD~\cite{IAD} and LIRA~\cite{LIRA}.}  % 图片标题
\label{fig:BDvisualization}  % 图片标签，用于交叉引用
\end{figure*}

\subsection{Backdoor Models and Target-Aligned Collapse}

We next examine whether these stationary output distributions reflect
semantically interpretable biases induced in the learned representations during training.
Modern deep networks are highly over-parameterized~\cite{StillRethinkingGeneralization},
and backdoored models are known to devote a disproportionate share
of this capacity to their target classes~\cite{liu2018sparecapacity}.
Backdoored models provide a particularly suitable testbed,
as the induced bias is semantically interpretable and has been shown
to be associated with internal representations rather than input-space geometry alone~\cite{ASurvey, LIRA}.
To assess whether such target-aligned biases are observable in the stationary output distributions,
we evaluate APEX across five backdoor attack methods
and three network architectures.
Representative model-level output distributions are shown in
Figure~\ref{fig:BDvisualization}.

\paragraph{Probing Backdoor Trigger Bias.} For benign models, the stationary output distributions induced by activation perturbation remain broadly spread across multiple classes.
In contrast, backdoored models exhibit a pronounced collapse:
under sufficiently large activation noise,
the vast majority of Monte Carlo runs predict the backdoor target class. More experimental settings are provided in
Appendix~\ref{app:bdsettings}. This target-aligned collapse indicates that,
once sample-dependent information is removed,
the learned representation of a backdoored model
is intrinsically biased toward the target class.

This distinction suggests that APEX provides an effective way to expose backdoored model biases. In particular, while optimized input- or parameter-level modifications can be designed or searched to induce target-aligned behavior, simple perturbations at these levels typically do not achieve this effect~\cite{MMBD,BD-BAN}.
Although input-level perturbations can also produce sample-independent outputs, their stationary output distributions under non-optimized noise do not concentrate on the backdoor target class (Figure~\ref{fig:bd_target_prob}).

\begin{figure}[h!]  % 'h!'表示尽可能在当前位置插入
\centering  % 图片居中
\includegraphics[width=0.6\linewidth]{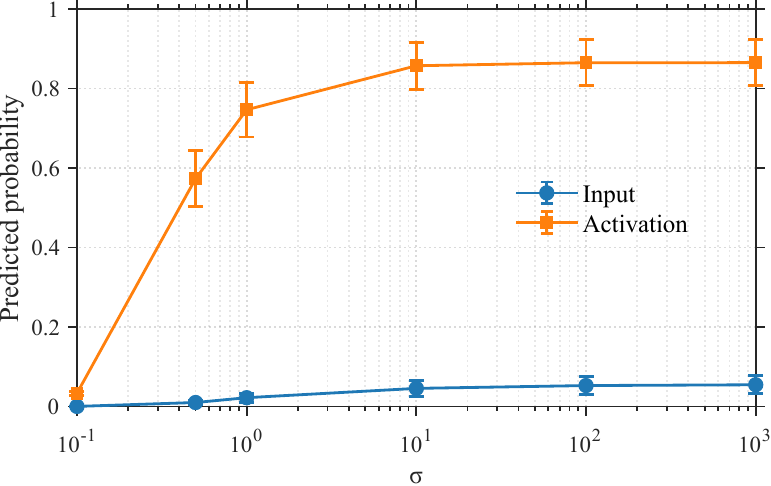}  
\caption{Average predicted probability of the backdoor target class
in the model-level output distribution under large perturbations. Results are obtained from ten CIFAR-100 backdoored ResNet-18 models.
The probability mass assigned to the target class is averaged across models.}
\label{fig:bd_target_prob}
\end{figure}

Figure~\ref{fig:benignscores} further shows that benign and backdoored models
exhibit systematically different levels of normalized entropy in their stationary output distributions,
consistent with the qualitative collapse patterns observed above.

\begin{figure*}[h!]
\centering
\includegraphics[width=\linewidth]{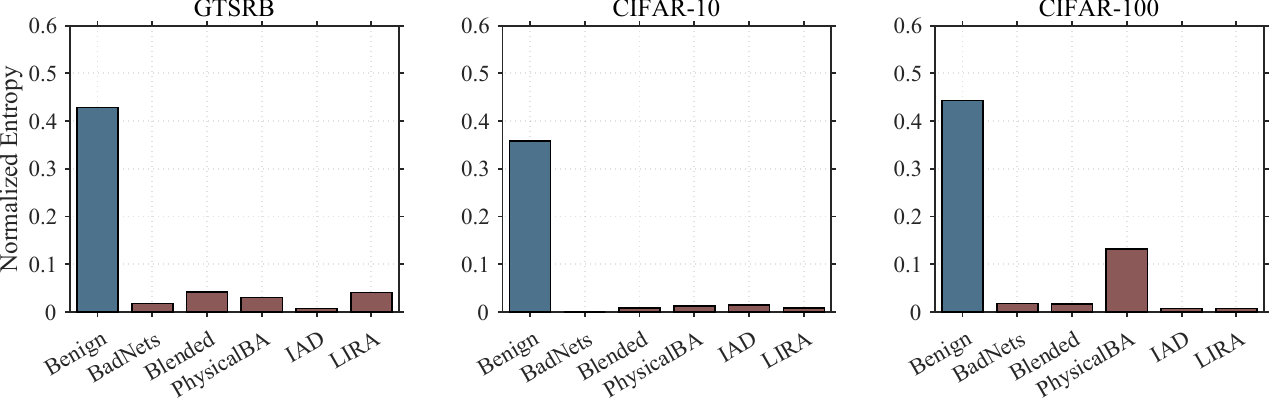}
\caption{
Average normalized entropy of benign and backdoored models across five attacks and three architecture–dataset pairs. Backdoored models consistently show entropy near zero, reflecting highly concentrated distributions.
}
\label{fig:benignscores}
\end{figure*}

\paragraph{Effect of Model Capacity.}
If the stationary output distribution under large activation noise reflects how representation space is allocated across classes, then models with smaller capacity should devote less representational volume to the backdoor target, resulting in weaker collapse and higher output entropy. To test this prediction, we analyze a family of CNNs with increasing depth trained under the same backdoor attack. As shown in Figure~\ref{fig:X2NET}, deeper models exhibit progressively stronger concentration of probability mass on the target class, while shallower models produce more dispersed stationary distributions.
These results suggest that activation perturbation may provide insight into how internal representation space is globally allocated.

\begin{figure}[h!]
\centering
\includegraphics[width=0.6\linewidth]{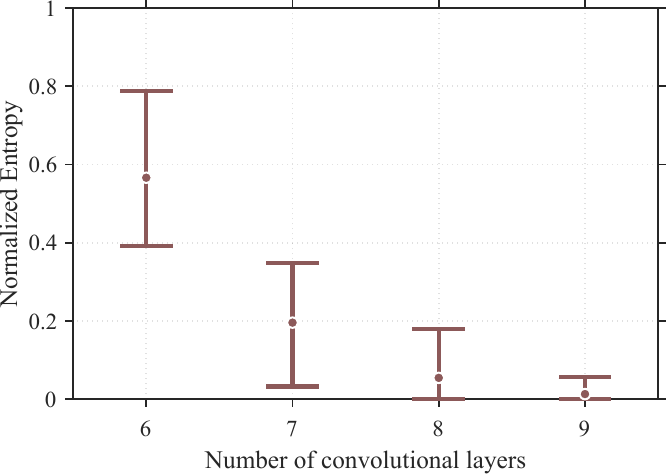}
\caption{
Normalized Entropy of backdoored models under the BadNets attack for X+2Net architectures of varying depth. Deeper models show more concentrated stationary output distributions.
}
\label{fig:X2NET}
\end{figure}

Interestingly, the trend is less pronounced in vision transformers(Appendix~\ref{app:transformer_extension}). Backdoored DeiT models exhibit partial stationary collapse, with distributions re-dispersing at larger noise levels. This suggests that different architectures may allocate representation space differently, attenuating collapse-based signals in transformers.

\section{Robustness Across Probing Configurations}
\label{sec:openquestions}

We examine how the behaviors revealed by APEX vary across different probing configurations.
These experiments address natural questions about sensitivity to design choices,
and help clarify which aspects of the observed phenomena are stable across reasonable variations.

% \section{Ablations and Robustness Analysis}
% \label{sec:openquestions}

% We conduct a series of ablation studies to verify that the behaviors revealed by APEX are intrinsic to the probing mechanism itself, rather than artifacts of confidence scaling, noise design, or specific architectural choices. All experiments follow the same training protocols described in Appendix~\ref{app:trainstats}.

\paragraph{Escape Noise vs. Confidence.}
We examine whether escape noise merely reflects model confidence accumulated during training. Across training epochs, escape-noise values remain largely unchanged despite a steady increase in average softmax confidence, indicating that escape noise is not driven by logit scaling but instead reflects local representation geometry.
Detailed results are reported in Appendix~\ref{app:confidence}.

\textbf{Noise Distribution Robustness.}
We explore whether the stationary output distribution depends on the specific noise distribution used for probing. Gaussian, uniform, and Laplace noise yield nearly identical stationary output distributions (Appendix~\ref{app:noisetype}), indicating that the effect is not tied to a particular noise model.

\textbf{Layer-wise Sensitivity.}
To gain insight into where stationary behavior emerges, we injected noise into individual ReLU layers of both clean and backdoored models and measured the average normalized entropy of the resulting output distributions. Similar to input-level perturbations, layer-wise activation perturbation can also be viewed as a special case of APEX, where noise is injected at a specific depth rather than across all layers. Perturbations to shallow layers produce sharply peaked distributions, whereas perturbations to deeper layers lead to substantially more dispersed outputs (Appendix~\ref{app:singlayer}).

\textbf{Activation Function Generality.}
Finally, we examine whether the observed target-aligned stationary behavior depends on the activation function. Models trained with Leaky ReLU and GELU exhibit the same sharp transition toward the backdoor target as activation noise increases, demonstrating that the stationary behavior arises from the perturbation mechanism itself rather than from ReLU-specific nonlinearities (Appendix~\ref{app:activationchoice}).

\section{Conclusion}

We present APEX as a unified inference-time probing paradigm that operates directly in activation space to reveal structural properties of learned representations. By inducing controlled perturbations, APEX exposes a clear behavioral transition from sample-dependent responses to model-dependent dynamics, which can be interpreted as reflecting local representational regularity and global training-induced biases. This unified view connects previously disparate phenomena, including sample regularity estimation, semantic transition analysis, and backdoor-induced target alignment, within a single framework. Together, these results position APEX as a lightweight yet powerful tool for analyzing modern neural networks.

\clearpage

\bibliography{references}

@inproceedings{chang2021privacy,
  title={On the privacy risks of algorithmic fairness},
  author={Chang, Hongyan and Shokri, Reza},
  booktitle=EuroSP,
  pages={292--303},
  year={2021},
  organization={IEEE}
}

@article{garg2023memorization,
  title={Memorization through the lens of curvature of loss function around samples},
  author={Garg, Isha and Ravikumar, Deepak and Roy, Kaushik},
  journal={arXiv preprint arXiv:2307.05831},
  year={2023}
}

@article{d2021tale,
  title={A tale of two long tails},
  author={D'souza, Daniel and Nussbaum, Zach and Agarwal, Chirag and Hooker, Sara},
  journal={arXiv preprint arXiv:2107.13098},
  year={2021}
}

@inproceedings{feldman2020does,
  title={Does learning require memorization? a short tale about a long tail},
  author={Feldman, Vitaly},
  booktitle={SIGACT Symposium on Theory of Computing},
  pages={954--959},
  year={2020},
organization={ACM}
}

@InProceedings{DeepRepresentationLearningonLong-TailedData,
author = {Liu, Jialun and Sun, Yifan and Han, Chuchu and Dou, Zhaopeng and Li, Wenhui},
title = {Deep Representation Learning on Long-Tailed Data: A Learnable Embedding Augmentation Perspective},
booktitle = {Proceedings of the IEEE/CVF Conference on Computer Vision and Pattern Recognition (CVPR)},
month = {June},
year = {2020}
}

@misc{li2025trustworthymachinelearningmemorization,
      title={Trustworthy Machine Learning via Memorization and the Granular Long-Tail: A Survey on Interactions, Tradeoffs, and Beyond}, 
      author={Qiongxiu Li and Xiaoyu Luo and Yiyi Chen and Johannes Bjerva},
      year={2025},
      eprint={2503.07501},
      archivePrefix={arXiv},
      primaryClass={cs.LG}
}

@inproceedings{DataPoisoningAttacksintheTrainingPhase,
  title={Data Poisoning Attacks in the Training Phase of Machine Learning Models: A Review},
  author={Mugdha Srivastava and Abhishek Kaushik and R{\'o}is{\'i}n Loughran and Kevin McDaid},
  booktitle={Irish Conference on Artificial Intelligence and Cognitive Science},
  year={2024}
}

@article{perturbationbased,
title = {Perturbation-based methods for explaining deep neural networks: A survey},
journal = {Pattern Recognition Letters},
volume = {150},
pages = {228-234},
year = {2021},
issn = {0167-8655},
doi = {https://doi.org/10.1016/j.patrec.2021.06.030},
author = {Maksims Ivanovs and Roberts Kadikis and Kaspars Ozols},
}

@inproceedings{whatneuralnetworkmem,
 author = {Feldman, Vitaly and Zhang, Chiyuan},
 booktitle = {Advances in Neural Information Processing Systems},
 pages = {2881--2891},
 publisher = {Curran Associates, Inc.},
 title = {What Neural Networks Memorize and Why: Discovering the Long Tail via Influence Estimation},
 volume = {33},
 year = {2020}
}

@InProceedings{C-score,
  title = 	 {Characterizing Structural Regularities of Labeled Data in Overparameterized Models},
  author =       {Jiang, Ziheng and Zhang, Chiyuan and Talwar, Kunal and Mozer, Michael C},
  booktitle = 	 {Proceedings of the 38th International Conference on Machine Learning},
  pages = 	 {5034--5044},
  year = 	 {2021},
  editor = 	 {Meila, Marina and Zhang, Tong},
  volume = 	 {139},
  series = 	 {Proceedings of Machine Learning Research},
  month = 	 {18--24 Jul},
  publisher =    {PMLR}
}

@ARTICLE{ASurvey,
  author={Li, Yiming and Jiang, Yong and Li, Zhifeng and Xia, Shu-Tao},
  journal={IEEE Transactions on Neural Networks and Learning Systems}, 
  title={Backdoor Learning: A Survey}, 
  year={2024},
  volume={35},
  number={1},
  pages={5-22},
  doi={10.1109/TNNLS.2022.3182979}}

@article{StillRethinkingGeneralization,
author = {Zhang, Chiyuan and Bengio, Samy and Hardt, Moritz and Recht, Benjamin and Vinyals, Oriol},
title = {Understanding deep learning (still) requires rethinking generalization},
year = {2021},
issue_date = {March 2021},
publisher = {Association for Computing Machinery},
address = {New York, NY, USA},
volume = {64},
number = {3},
issn = {0001-0782},
journal = {Commun. ACM},
month = feb,
pages = {107–115},
numpages = {9}
}

@inproceedings{liu2018sparecapacity,
  title={Fine-pruning: Defending against backdooring attacks on deep neural networks},
  author={Liu, Kang and Dolan-Gavitt, Brendan and Garg, Siddharth},
  booktitle={International symposium on research in attacks, intrusions, and defenses},
  pages={273--294},
  year={2018},
  organization={Springer}
}

@inproceedings{backdoorbox,
  title        = {BackdoorBox: A Python Toolbox for Backdoor Learning},
  author={Yiming Li and Mengxi Ya and Yang Bai and Yong Jiang and Shu-Tao Xia},
  booktitle    = {ICLR 2023 Workshop on Backdoor Attacks and Defenses in Machine Learning},
  year         = {2023},
  note         = {arXiv:2302.01762},
}

@ARTICLE{badnets,
  author={Gu, Tianyu and Liu, Kang and Dolan-Gavitt, Brendan and Garg, Siddharth},
  journal={IEEE Access}, 
  title={BadNets: Evaluating Backdooring Attacks on Deep Neural Networks}, 
  year={2019},
  volume={7},
  number={},
  pages={47230-47244},
  doi={10.1109/ACCESS.2019.2909068}}

@misc{Blended,
      title={Targeted Backdoor Attacks on Deep Learning Systems Using Data Poisoning}, 
      author={Xinyun Chen and Chang Liu and Bo Li and Kimberly Lu and Dawn Song},
      year={2017},
      eprint={1712.05526},
      archivePrefix={arXiv},
      primaryClass={cs.CR}
}

@inproceedings{physical,
  title     = {Backdoor Attack in the Physical World},
  author={Yiming Li and Tongqing Zhai and Yong Jiang and Zhifeng Li and Shu-Tao Xia},
  booktitle = {ICLR 2021 Workshop on RobustML},
  journal   = {arXiv},
  volume    = {2104.02361},
  year      = {2021}
}

@inproceedings{IAD,
author = {Nguyen, Tuan Anh and Tran, Tuan Anh},
title = {Input-aware dynamic backdoor attack},
year = {2020},
isbn = {9781713829546},
publisher = {Curran Associates Inc.},
address = {Red Hook, NY, USA},
booktitle = {Proceedings of the 34th International Conference on Neural Information Processing Systems},
articleno = {291},
numpages = {11},
location = {Vancouver, BC, Canada},
series = {NIPS '20}
}

@inproceedings{LIRA,
  title={LIRA: Learnable, Imperceptible and Robust Backdoor Attacks},
  author={Khoa Doan and Yingjie Lao and Weijie Zhao and Ping Li},
  booktitle={ICCV},
  pages={11946--11956},
  year={2021},
  doi={10.1109/ICCV48922.2021.01175}
}

@inproceedings{
novak2018sensitivity,
title={Sensitivity and Generalization in Neural Networks: an Empirical Study},
author={Roman Novak and Yasaman Bahri and Daniel A. Abolafia and Jeffrey Pennington and Jascha Sohl-Dickstein},
booktitle={International Conference on Learning Representations},
year={2018}
}

@article{LIUjacobian,
title = {Jacobian norm with Selective Input Gradient Regularization for interpretable adversarial defense},
journal = {Pattern Recognition},
volume = {145},
pages = {109902},
year = {2024},
issn = {0031-3203},
doi = {https://doi.org/10.1016/j.patcog.2023.109902},
author = {Deyin Liu and Lin Yuanbo Wu and Bo Li and Farid Boussaid and Mohammed Bennamoun and Xianghua Xie and Chengwu Liang}
}

@inproceedings{rankdiminishing,
author = {Feng, Ruili and Zheng, Kecheng and Huang, Yukun and Zhao, Deli and Jordan, Michael and Zha, Zheng-Jun},
title = {Rank diminishing in deep neural networks},
year = {2022},
isbn = {9781713871088},
publisher = {Curran Associates Inc.},
address = {Red Hook, NY, USA},
booktitle = {Proceedings of the 36th International Conference on Neural Information Processing Systems},
articleno = {2395},
numpages = {12},
location = {New Orleans, LA, USA},
series = {NIPS '22}
}

@INPROCEEDINGS{decision_karimi,
  author={Karimi, Hamid and Derr, Tyler},
  booktitle={2022 21st IEEE International Conference on Machine Learning and Applications (ICMLA)}, 
  title={Decision Boundaries of Deep Neural Networks}, 
  year={2022},
  volume={},
  number={},
  pages={1085-1092},
  doi={10.1109/ICMLA55696.2022.00179}}

@inproceedings{fawzi-geometry,
author = {Fawzi, Alhussein and Moosavi-Dezfooli, Seyed-Mohsen and Frossard, Pascal and Soatto, Stefano},
year = {2018},
month = {06},
pages = {3762-3770},
title = {Empirical Study of the Topology and Geometry of Deep Networks},
doi = {10.1109/CVPR.2018.00396}
}

@inproceedings{inputperturbation,
author = {Fawzi, Alhussein and Moosavi-Dezfooli, Seyed-Mohsen and Frossard, Pascal},
title = {Robustness of classifiers: from adversarial to random noise},
year = {2016},
isbn = {9781510838819},
publisher = {Curran Associates Inc.},
address = {Red Hook, NY, USA},
booktitle = {Proceedings of the 30th International Conference on Neural Information Processing Systems},
pages = {1632–1640},
numpages = {9},
location = {Barcelona, Spain},
series = {NIPS'16}
}

@misc{shahani2025noiseinjectionsystemicallydegrades,
      title={Noise Injection Systemically Degrades Large Language Model Safety Guardrails}, 
      author={Prithviraj Singh Shahani and Kaveh Eskandari Miandoab and Matthias Scheutz},
      year={2025},
      eprint={2505.13500},
      archivePrefix={arXiv},
      primaryClass={cs.CL}
}

@InProceedings{perturb-neuron-davis21a,
  title = 	 {Catformer: Designing Stable Transformers via Sensitivity Analysis},
  author =       {Davis, Jared Q and Gu, Albert and Choromanski, Krzysztof and Dao, Tri and Re, Christopher and Finn, Chelsea and Liang, Percy},
  booktitle = 	 {Proceedings of the 38th International Conference on Machine Learning},
  pages = 	 {2489--2499},
  year = 	 {2021},
  editor = 	 {Meila, Marina and Zhang, Tong},
  volume = 	 {139},
  series = 	 {Proceedings of Machine Learning Research},
  month = 	 {18--24 Jul},
  publisher =    {PMLR}
}

@article{Sharma2023InvestigatingWD,
  title={Investigating Weight-Perturbed Deep Neural Networks with Application in Iris Presentation Attack Detection},
  author={Renu Sharma and Redwan Sony and Arun Ross},
  journal={2024 IEEE/CVF Winter Conference on Applications of Computer Vision Workshops (WACVW)},
  year={2024},
  pages={1082-1091}
}

@inproceedings{
li2018measuring,
title={Measuring the Intrinsic Dimension of Objective Landscapes},
author={Chunyuan Li and Heerad Farkhoor and Rosanne Liu and Jason Yosinski},
booktitle={International Conference on Learning Representations},
year={2018}
}

@INPROCEEDINGS{Internal-representation-tishby,
  author={Tishby, Naftali and Zaslavsky, Noga},
  booktitle={2015 IEEE Information Theory Workshop (ITW)}, 
  title={Deep learning and the information bottleneck principle}, 
  year={2015},
  volume={},
  number={},
  pages={1-5},
  doi={10.1109/ITW.2015.7133169}}

@InProceedings{ford2019adversarial,
  title = 	 {Adversarial Examples Are a Natural Consequence of Test Error in Noise},
  author =       {Gilmer, Justin and Ford, Nicolas and Carlini, Nicholas and Cubuk, Ekin},
  booktitle = 	 {Proceedings of the 36th International Conference on Machine Learning},
  pages = 	 {2280--2289},
  year = 	 {2019},
  editor = 	 {Chaudhuri, Kamalika and Salakhutdinov, Ruslan},
  volume = 	 {97},
  series = 	 {Proceedings of Machine Learning Research},
  month = 	 {09--15 Jun},
  publisher =    {PMLR}
}

@inproceedings{cifar,
  title={Learning multiple layers of features from tiny images},
  author={Krizhevsky, A. et al.},
  booktitle={Univ. Toronto Tech. Rep.},
  year={2009}
}

@inproceedings{GTSRB,
  author = {Stallkamp, J. et al.},
  title = {The German Traffic Sign Recognition Benchmark: A Multi-Class Classification Competition},
  booktitle = {IEEE IJCNN},
  pages = {1453--1460},
  year = {2011}
}

@inproceedings{resnet,
  title={Deep Residual Learning for Image Recognition},
  author={He, K. et al.},
  booktitle={Proc. CVPR},
  pages={770--778},
  year={2016}
}

@article{ILSVRC15,
Author = {Olga Russakovsky and Jia Deng and Hao Su and Jonathan Krause and Sanjeev Satheesh and Sean Ma and Zhiheng Huang and Andrej Karpathy and Aditya Khosla and Michael Bernstein and Alexander C. Berg and Li Fei-Fei},
Title = {{ImageNet Large Scale Visual Recognition Challenge}},
Year = {2015},
journal   = {International Journal of Computer Vision (IJCV)},
doi = {10.1007/s11263-015-0816-y},
volume={115},
number={3},
pages={211-252}
}

@InProceedings{Simonyan15,
  author       = "Karen Simonyan and Andrew Zisserman",
  title        = "Very Deep Convolutional Networks for Large-Scale Image Recognition",
  booktitle    = "International Conference on Learning Representations",
  year         = "2015",
}

@INPROCEEDINGS{densenet,
  author={Huang, Gao and Liu, Zhuang and Van Der Maaten, Laurens and Weinberger, Kilian Q.},
  booktitle={2017 IEEE Conference on Computer Vision and Pattern Recognition (CVPR)}, 
  title={Densely Connected Convolutional Networks}, 
  year={2017},
  volume={},
  number={},
  pages={2261-2269},
  doi={10.1109/CVPR.2017.243}}

@inproceedings{
BD-BAN,
title={{BAN}: Detecting Backdoors Activated by Neuron Noise},
author={Xiaoyun Xu and Zhuoran Liu and Stefanos Koffas and Shujian Yu and Stjepan Picek},
booktitle={The Thirty-eighth Annual Conference on Neural Information Processing Systems},
year={2024},
url={https://openreview.net/forum?id=asYYSzL4N5}
}

@INPROCEEDINGS {MMBD,
author = { Wang, Hang and Xiang, Zhen and Miller, David J. and Kesidis, George },
booktitle = { 2024 IEEE Symposium on Security and Privacy (SP) },
title = {{ MM-BD: Post-Training Detection of Backdoor Attacks with Arbitrary Backdoor Pattern Types Using a Maximum Margin Statistic }},
year = {2024},
volume = {},
ISSN = {},
pages = {1994-2012},
doi = {10.1109/SP54263.2024.00015},
url = {https://doi.ieeecomputersociety.org/10.1109/SP54263.2024.00015},
publisher = {IEEE Computer Society},
address = {Los Alamitos, CA, USA},
month =May}
\bibliographystyle{unsrtnat}

%%%%%%%%%%%%%%%%%%%%%%%%%%%%%%%%%%%%%%%%%%%%%%%%%%%%%%%%%%%%%%%%%%%%%%%%%%%%%%%
%%%%%%%%%%%%%%%%%%%%%%%%%%%%%%%%%%%%%%%%%%%%%%%%%%%%%%%%%%%%%%%%%%%%%%%%%%%%%%%
% APPENDIX
%%%%%%%%%%%%%%%%%%%%%%%%%%%%%%%%%%%%%%%%%%%%%%%%%%%%%%%%%%%%%%%%%%%%%%%%%%%%%%%
%%%%%%%%%%%%%%%%%%%%%%%%%%%%%%%%%%%%%%%%%%%%%%%%%%%%%%%%%%%%%%%%%%%%%%%%%%%%%%%
\newpage
\appendix
\onecolumn
\section{Memorization and long-tailed distribution}\label{app:longtailed}

Beyond its impact on generation performance, memorization in deep neural networks has far-reaching implications for safety and security considerations, including privacy leakage and fairness disparities \cite{feldman2020does,chang2021privacy,garg2023memorization}.

The memorization score \cite{whatneuralnetworkmem} measures the extent to which a trained model's prediction on a particular example depends on whether that example was included during training. Let $\mathcal{A}$ be a (possibly randomized) learning algorithm and $D_{\mathrm{tr}}$ a training set. For an example $i=(x_i,y_i)\in D_{\mathrm{tr}}$, its (label) memorization score is defined via a leave-one-out comparison:
\begin{align}\label{eq:mem}
\mathrm{mem}(\mathcal{A}, D_{\mathrm{tr}}, i)
&= \Pr_{h \leftarrow \mathcal{A}(D_{\mathrm{tr}})} \big[ h(x_i)=y_i \big]
- \Pr_{h \leftarrow \mathcal{A}(D_{\mathrm{tr}}\setminus\{i\})} \big[ h(x_i)=y_i \big],
\end{align}
where $D_{\mathrm{tr}}\setminus\{i\}$ denotes the dataset obtained by removing $i$ from $D_{\mathrm{tr}}$.

Intuitively, $\mathrm{mem}(\cdot)$ captures the marginal influence of a single training point on the model's correctness at that point. A larger value indicates that the model relies more heavily on $i$ to predict $y_i$, which is commonly interpreted as a signal of elevated privacy exposure since the model's behavior changes noticeably when $i$ is present versus absent. Empirically, high memorization scores tend to concentrate on atypical or hard-to-learn instances, whereas low scores are more often associated with typical, easy examples across datasets \cite{whatneuralnetworkmem}, aligning with qualitative notions of example difficulty. \citet{d2021tale} showing that the long tail itself is heterogeneous: atypical-but-valid samples correspond to reducible (epistemic) uncertainty, whereas noisy samples induce irreducible (aleatoric) uncertainty and exhibit distinct learning dynamics. A recent survey echoes that conflating these different sources of tail behavior obscures the interpretation of memorization, whose implications for generalization, privacy, and robustness depend critically on whether memorization arises from atypicality or noise \cite{li2025trustworthymachinelearningmemorization}.

\section{Decomposition of Activation-Perturbed Activations}
\label{app:DecompositionProP}

We detail the layerwise decomposition used in Theorem~\ref{thm:activation-decomposition}.

\begin{proof}[Decomposition after ReLU with uniformly bounded residual]
Assume noise is added after each ReLU:
\[
\tilde{a}_1(x;\sigma)=\mathrm{ReLU}(W_1x+b_1)+\sigma\xi_1,
\]
\[
\tilde{a}_\ell(x;\sigma)=\mathrm{ReLU}(W_\ell\tilde{a}_{\ell-1}(x;\sigma)+b_\ell)+\sigma\xi_\ell,
\quad \ell=2,\dots,L.
\]

We show by induction that for each $\ell$ there exist $v_\ell$ (depending on $(W,b,\xi)$ but not on $x$)
and $r_\ell(x;\sigma)$ such that
\[
\tilde{a}_\ell(x;\sigma)=\sigma v_\ell + r_\ell(x;\sigma),
\qquad
\|r_\ell(x;\sigma)\| \le B_\ell(R;W,b),
\]
uniformly over $x\in\mathcal{X}=\{x:\|x\|\le R\}$ and $\sigma>0$.

\textbf{Base case ($\ell=1$).}
Define
\[
v_1 := \xi_1,\qquad r_1(x;\sigma):=\mathrm{ReLU}(W_1x+b_1).
\]
Then $\tilde{a}_1(x;\sigma)=\sigma v_1+r_1(x;\sigma)$ and $\|r_1(x;\sigma)\|\le \|W_1\|R+\|b_1\|$.

\textbf{Inductive step.}
Assume for layer $\ell-1$ that
\[
\tilde{a}_{\ell-1}(x;\sigma)=\sigma v_{\ell-1}+r_{\ell-1}(x;\sigma).
\]
Then
\[
W_\ell\tilde{a}_{\ell-1}(x;\sigma)+b_\ell
=\sigma W_\ell v_{\ell-1}+\big(W_\ell r_{\ell-1}(x;\sigma)+b_\ell\big).
\]
Applying Lemma~\ref{lem:relu-residual} with
$a=W_\ell v_{\ell-1}$ and $d=W_\ell r_{\ell-1}(x;\sigma)+b_\ell$, we obtain
\[
\mathrm{ReLU}(\sigma W_\ell v_{\ell-1}+W_\ell r_{\ell-1}(x;\sigma)+b_\ell)
=\sigma\,\mathrm{ReLU}(W_\ell v_{\ell-1})+\Delta_\ell(x;\sigma),
\]
where
\[
\|\Delta_\ell(x;\sigma)\|\le \|W_\ell r_{\ell-1}(x;\sigma)+b_\ell\|.
\]
Therefore,
\[
\tilde{a}_\ell(x;\sigma)
=\sigma\,\mathrm{ReLU}(W_\ell v_{\ell-1})+\Delta_\ell(x;\sigma)+\sigma\xi_\ell
=\sigma v_\ell + r_\ell(x;\sigma),
\]
by defining
\[
v_\ell := \mathrm{ReLU}(W_\ell v_{\ell-1})+\xi_\ell,
\qquad
r_\ell(x;\sigma):=\Delta_\ell(x;\sigma).
\]
Moreover,
\[
\|r_\ell(x;\sigma)\|
\le \|W_\ell r_{\ell-1}(x;\sigma)+b_\ell\|
\le \|W_\ell\|\,\|r_{\ell-1}(x;\sigma)\|+\|b_\ell\|.
\]
Iterating this recursion yields the bound $\|r_\ell(x;\sigma)\|\le B_\ell(R;W,b)$ from
Appendix~\ref{app:boundness}, which is independent of $\sigma$.
\end{proof}

\section{Uniform boundedness of the residual term}
\label{app:boundness}
Fix an induced operator norm $\|\cdot\|$ and assume inputs lie in a bounded
set $\mathcal{X}=\{x:\|x\|\le R\}$ with $R<\infty$.
Let $M_\ell:=\|W_\ell\|$ and $b_\ell^\star:=\|b_\ell\|$.
From the inductive step in Appendix~\ref{app:DecompositionProP}, we have
\[
\|r_\ell(x;\sigma)\|
\le \|W_\ell r_{\ell-1}(x;\sigma)+b_\ell\|
\le M_\ell \|r_{\ell-1}(x;\sigma)\|+b_\ell^\star.
\]
With $r_1(x;\sigma)=\mathrm{ReLU}(W_1x+b_1)$ and $\|r_1(x;\sigma)\|\le M_1R+b_1^\star$,
recursively we obtain
\[
\|r_\ell(x;\sigma)\| \le B_\ell(R;W,b)
:= \Big(\prod_{j=1}^{\ell} M_j\Big) R
   + \sum_{k=1}^{\ell-1}\Big(\prod_{j=k+1}^{\ell} M_j\Big) b_k^\star+ b_l^\star.
\]
Hence $r_\ell(x;\sigma)$ is bounded on $\mathcal{X}$ by $B_\ell(R;W,b)$, which depends
on $(W,b)$ and $R$ but is independent of $\sigma$.

\section{Stationary prediction distribution at large noise}
\label{app:GPD}
We formalize the large-noise regime under noise resampling at inference time.

Write the logits as
\[
s(x;\sigma)=U\tilde{a}_L(x;\sigma)+c
=\sigma\,Uv_L + e(x;\sigma),
\qquad
e(x;\sigma):=Ur_L(x;\sigma)+c.
\]
Let $C:=\|U\|_\infty B_L(R;W,b)+\|c\|_\infty$, so that
$\|e(x;\sigma)\|_\infty\le C$ for all $x\in\mathcal{X}$ and $\sigma>0$
(Appendix~\ref{app:boundness}).

\paragraph{Conditional stabilization for a fixed noise draw.}
Condition on a fixed draw of the injected noises (hence a fixed $v_L$).
Let $j^\star=\arg\max_i (Uv_L)_i$ and define the top-1 margin
\[
\delta := (Uv_L)_{j^\star}-\max_{i\ne j^\star}(Uv_L)_i.
\]
Assume $\delta>0$. If
\[
\sigma>\sigma^\star:=\frac{2C}{\delta},
\]
then for all $x\in\mathcal{X}$,
\[
\arg\max_i s_i(x;\sigma)=\arg\max_i (Uv_L)_i = j^\star.
\]
Indeed, for any $i\ne j^\star$,
\[
s_{j^\star}(x;\sigma)-s_i(x;\sigma)
=\sigma\big((Uv_L)_{j^\star}-(Uv_L)_i\big) + \big(e_{j^\star}(x;\sigma)-e_i(x;\sigma)\big)
\ge \sigma\delta - 2C > 0.
\]

\paragraph{Stationary limit under noise resampling.}
Since noises are resampled at each inference, $v_L$ varies randomly across runs.
The previous result implies that as $\sigma\to\infty$, the prediction distribution
converges to the distribution of $\arg\max_i (Uv_L)_i$, which depends on the trained
model $(W,b,U,c)$ and the noise law, but not on the input $x$.

\section{Lemma: ReLU residual bound}
\label{app:lemma}
\begin{lemma}
\label{lem:relu-residual}
For any vectors $a,d\in\mathbb{R}^m$ and scalar $\sigma>0$, define
\[
\Delta := \mathrm{ReLU}(\sigma a+d)-\sigma\,\mathrm{ReLU}(a).
\]
Then
\[
\mathrm{ReLU}(\sigma a+d)=\sigma\,\mathrm{ReLU}(a)+\Delta,
\qquad
\|\Delta\|\le \|d\|.
\]
\end{lemma}

\begin{proof}
Since ReLU is positively homogeneous for $\alpha\ge 0$,
$\mathrm{ReLU}(\sigma a)=\sigma\,\mathrm{ReLU}(a)$.
As ReLU is $1$-Lipschitz,
\[
\|\mathrm{ReLU}(\sigma a+d)-\mathrm{ReLU}(\sigma a)\|\le \|d\|.
\]
Combining the two identities gives $\|\Delta\|\le \|d\|$.
\end{proof}

\clearpage
\section{Experimental Setup}
\label{app:trainstats}

\paragraph{Models and Datasets.}
We evaluate our method on CIFAR-10, CIFAR-100~\cite{cifar}, GTSRB~\cite{GTSRB}, and ImageNet-1K~\cite{ILSVRC15} using a diverse set of vision architectures. Our experiments cover convolutional models including ResNet-18, ResNet-50~\cite{resnet}, VGG-11~\cite{Simonyan15}, DenseNet~\cite{densenet}, Inception, and a family of X+2Net architectures with varying depth, as well as the transformer-based DeiT-Small model. For ImageNet-1K experiments, ResNet-50 is initialized from ImageNet-1K pretrained weights and fine-tuned on the target dataset; all other models and datasets are trained from scratch.

\paragraph{Training Protocol.}
Unless otherwise stated, all convolutional models are trained using stochastic gradient descent (SGD) with momentum.
We adopt a largely unified training configuration across datasets, with minor dataset-specific adjustments. Experiments on CIFAR-10, CIFAR-100, and GTSRB~\cite{cifar,GTSRB} use SGD with momentum $0.9$.
The following configuration is used:
\begin{itemize}
    \item Batch size: 256
    \item Number of dataloader workers: 4
    \item Learning rate: $0.1$
    \item Weight decay: $5\times10^{-4}$
    \item Learning rate scheduler: step decay with milestones at epochs 50 and 80
    \item Learning rate decay factor: $\gamma = 0.1$
    \item Total training epochs: 100
\end{itemize}

\paragraph{Inference-Time Perturbation and Monte Carlo Estimation.}
All perturbation experiments are conducted strictly at inference time, with model parameters fixed.
To estimate the output distribution we perform 1{,}000 Monte Carlo forward passes per input.
For ImageNet-1K, due to computational constraints, we perform 100 Monte Carlo forward passes per input.

\clearpage

\section{More Results on Random-Label Imprints}
\label{app:randomlabel}
Across more settings, increasing label randomness leads to a monotonic decrease in escape noise, 
indicating progressively fragmented local decision regions.
\begin{figure}[h]
    \centering
    \begin{subfigure}[t]{0.35\textwidth}
        \centering
        \includegraphics[width=\linewidth]{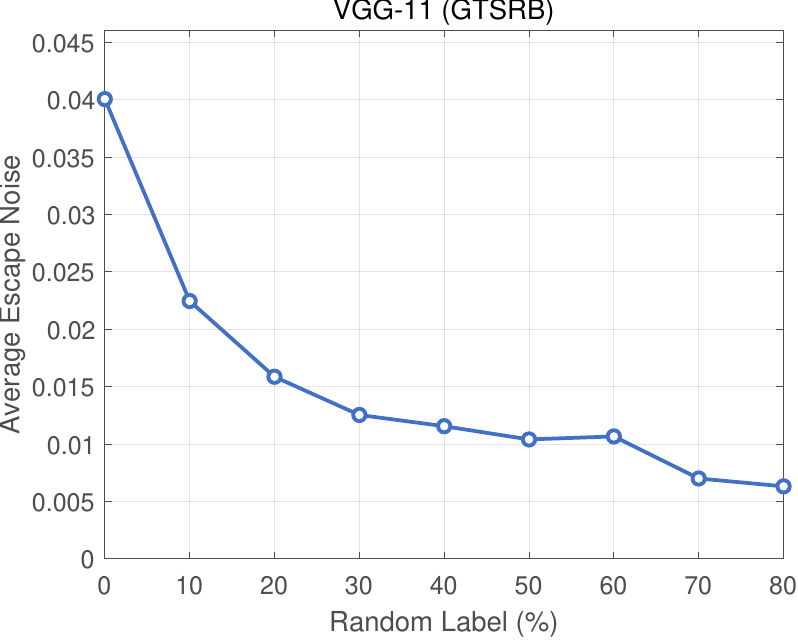}
    \end{subfigure}\hspace{0.04\textwidth}%
    \begin{subfigure}[t]{0.35\textwidth}
        \centering
        \includegraphics[width=\linewidth]{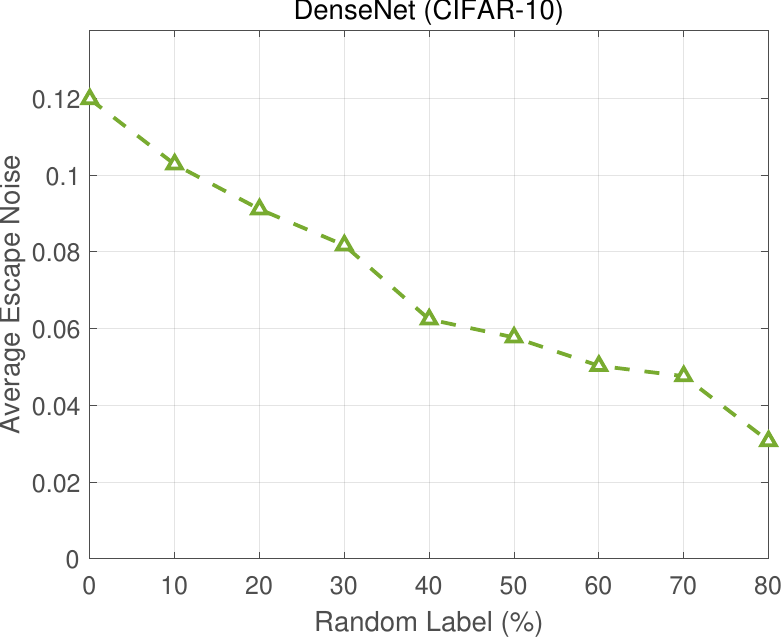}
    \end{subfigure}

    \caption{Average escape noise for models with different percentages of samples randomly labeled across two settings.}
    \label{fig:randomlabel13}
\end{figure}

\section{Random-Label Ablations under Input and Parameter Perturbations}
\label{app:randomlabel_ablation}

We report corresponding ablation experiments under input- and parameter-level perturbations,
using the same datasets, architectures, and random-label settings.

\begin{figure*}[h]
    \centering
    \begin{subfigure}{0.85\linewidth}
        \centering
        \includegraphics[width=\linewidth]{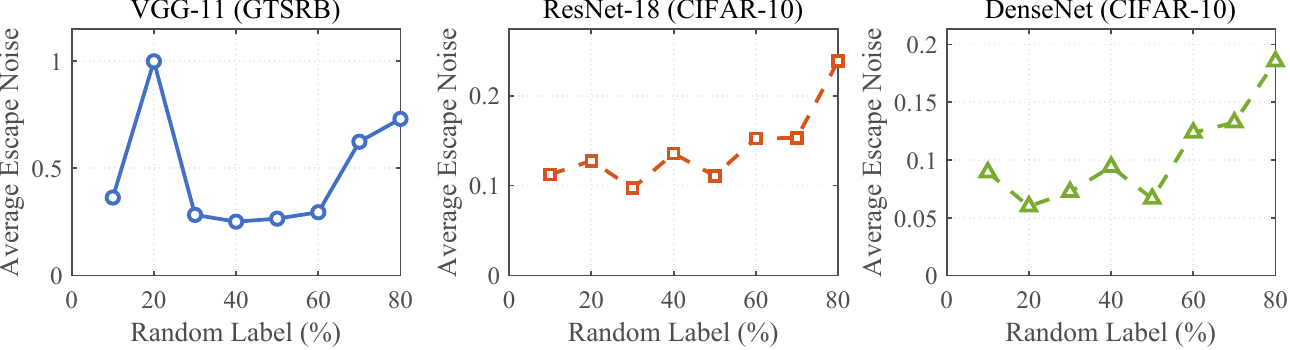}
        \caption{Input perturbation.}
        \label{fig:randomlabel_input}
    \end{subfigure}

    \vspace{0.5em}

    \begin{subfigure}{0.85\linewidth}
        \centering
        \includegraphics[width=\linewidth]{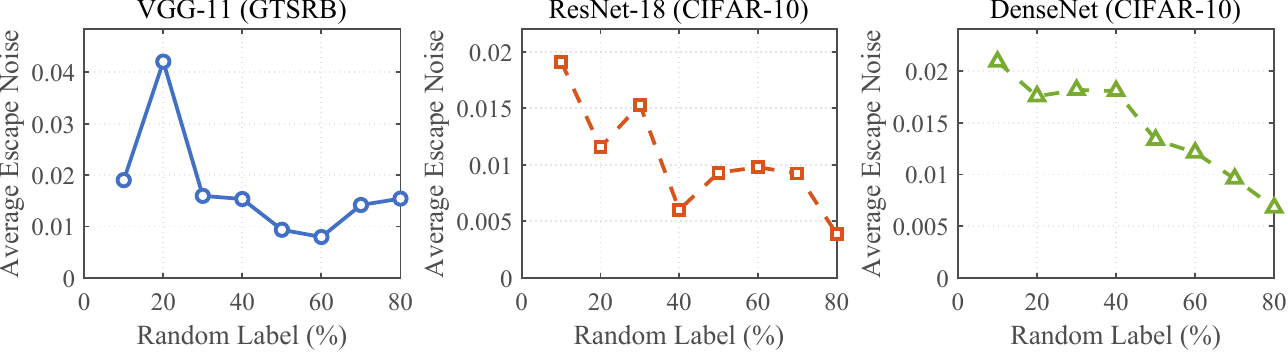}
        \caption{Parameter perturbation.}
        \label{fig:randomlabel_weight}
    \end{subfigure}

    \caption{
    Average escape noise under different random-label ratios for
    \emph{input-} (top) and \emph{parameter-level} (bottom) perturbations.
    Experimental settings follow Figure~\ref{fig:randomlabel}.
    Neither perturbation exhibits a consistent monotonic decrease of escape noise
    as label randomness increases.
    }
    \label{fig:randomlabel_ablation}
\end{figure*}

\clearpage

\section{Implementation Details of Input and Parameter Perturbations}
\label{app:perturb_impl}

We briefly describe the implementation of input- and parameter-level perturbations used in our experiments. To ensure a controlled comparison across perturbation mechanisms, we follow the same non-optimized, inference-time perturbation protocol used by APEX.

\paragraph{Input Perturbation.}
For input perturbation, we add i.i.d. Gaussian noise to the input image at inference time:
\[
x' = \mathrm{clip}(x + \sigma \epsilon, 0, 1), \quad \epsilon \sim \mathcal{N}(0, I),
\]
where $\sigma$ denotes the noise standard deviation.

\paragraph{Parameter Perturbation.}
For parameter perturbation, we add i.i.d. Gaussian noise to all floating-point parameters of the network:
\[
W' = W + \sigma \epsilon, \quad \epsilon \sim \mathcal{N}(0, I),
\]
where the same $\sigma$ is applied uniformly across layers. In all experiments, both perturbation types are non-optimized and applied independently across repeated evaluations.

\clearpage

\section{Additional Results on Perturbation Stability and Sample Regularity}
\label{app:EN_regularity_detail}

In the main text, we focus on activation perturbation as a general framework
for probing sample regularity, and note that input perturbation constitutes
a constrained special case.
Here, we report corresponding input- and activation-level results across
datasets and architectures to corroborate this relationship. The consistency score and memorization score form approximately complementary measures
of sample regularity, satisfying $\text{C-score} + \text{Mem-score} \approx 1$~\cite{C-score}.
Therefore, we only report results in terms of C-score.

Specifically, we evaluate prediction stability under different noise magnitudes $\sigma$
on ImageNet and CIFAR-100, with ResNet-50 and inception model respectively. All stability curves are obtained by averaging Monte Carlo prediction frequencies
within bins of the corresponding regularity scores.

\begin{figure*}[h!]
    \captionsetup{width=\textwidth}
    \centering
    % ------------------- CIFAR-100: activation -------------------
    \begin{subfigure}[b]{0.45\textwidth}
        \centering
        \includegraphics[width=\textwidth]{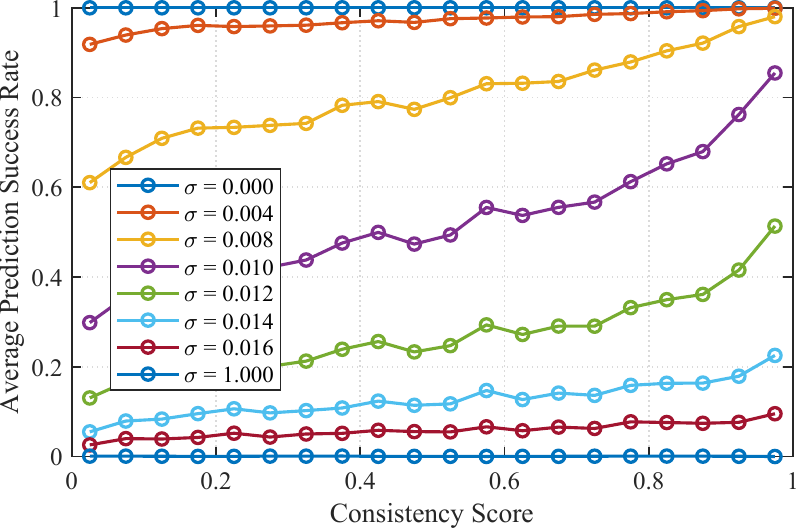}
        \caption{CIFAR-100: activation perturbation stability.}
        \label{fig:appendix_c100_act}
    \end{subfigure}
    \hfill
    % ------------------- ImageNet: activation ------------------
    \begin{subfigure}[b]{0.45\textwidth}
        \centering
        \includegraphics[width=\textwidth]{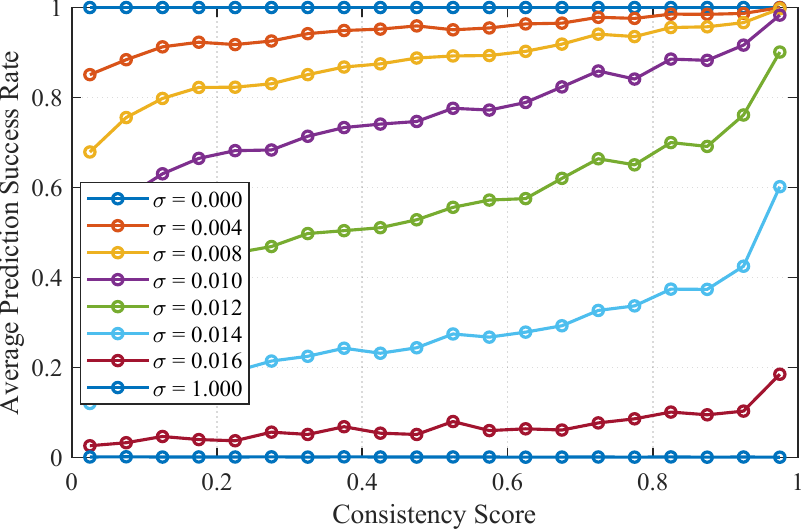}
        \caption{ImageNet: activation perturbation stability.}
        \label{fig:appendix_in_act}
    \end{subfigure}
    
    \vspace{0.4cm}
    
    % ------------------- CIFAR-100: input -------------------
    \begin{subfigure}[b]{0.45\textwidth}
        \centering
        \includegraphics[width=\textwidth]{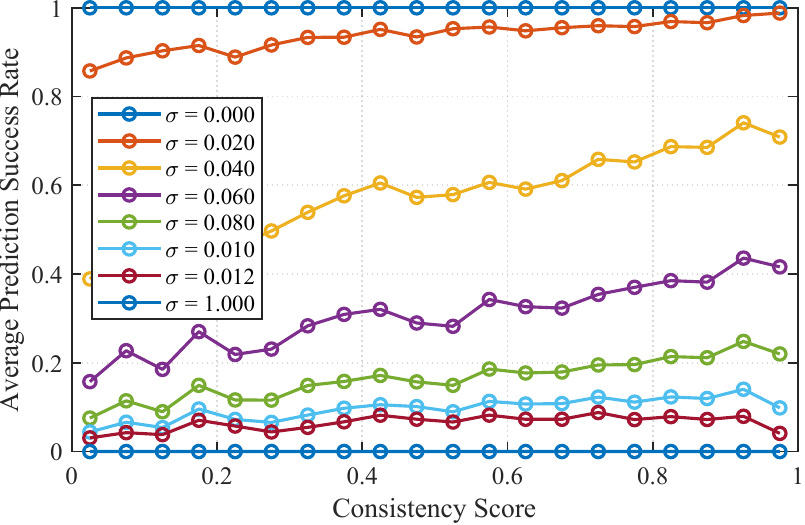}
        \caption{CIFAR-100: input perturbation stability (constrained case).}
        \label{fig:appendix_c100_input}
    \end{subfigure}
    \hfill
    % ------------------- ImageNet: input -------------------
    \begin{subfigure}[b]{0.45\textwidth}
        \centering
        \includegraphics[width=\textwidth]{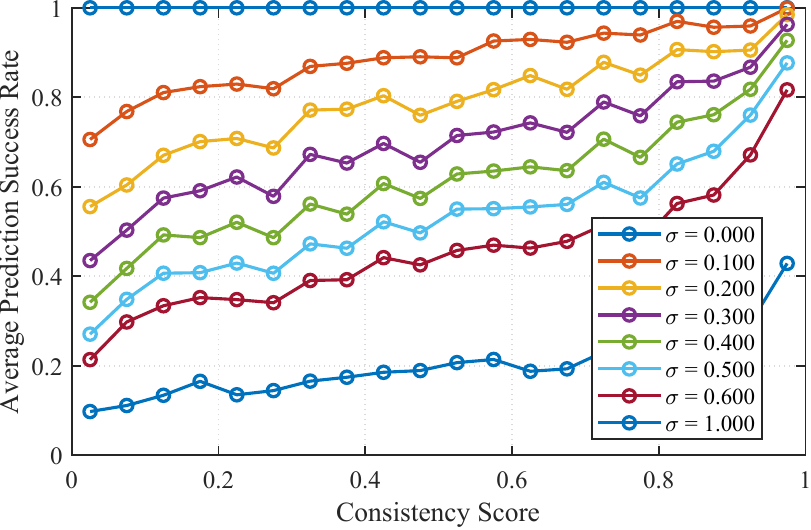}
        \caption{ImageNet: input perturbation stability (constrained case).}
        \label{fig:appendix_in_input}
    \end{subfigure}

    \caption{
    Additional sample-level stability results on ImageNet and CIFAR-100.
    Top: activation perturbation, serving as the general framework.
    Bottom: input perturbation, corresponding to a constrained special case.
    Across datasets, prediction stability exhibits consistent correlations
    with consistency score, supporting the interpretation of escape noise as a reliable indicator of sample regularity.
    }
    \label{fig:appendix_LPD}
\end{figure*}
\clearpage

\section{Model-Level Output Distributions}
\label{app:stationary}

\begin{figure*}[h!]
    \captionsetup{width=\textwidth}
    \centering
    % ------------------- CIFAR-100: activation -------------------
    \begin{subfigure}[b]{0.45\textwidth}
        \centering
        \includegraphics[width=\textwidth]{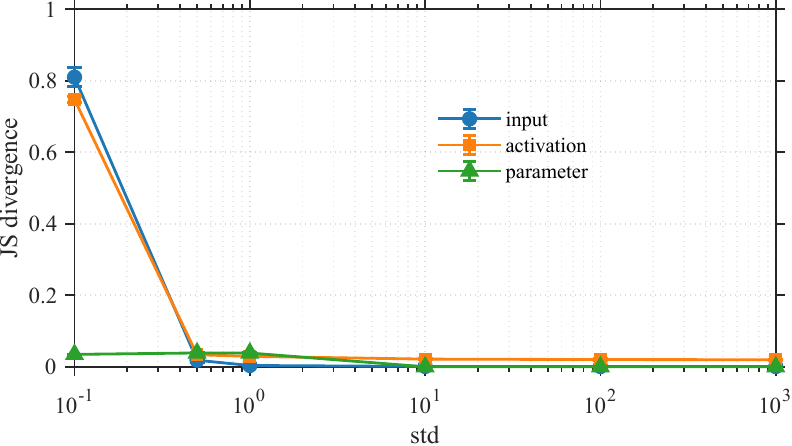}
        \caption{JS divergence between output distributions from different inputs for ResNet-18 models.}
        \label{fig:appendix_c100_act}
    \end{subfigure}
    \hfill
    % ------------------- ImageNet: activation ------------------
    \begin{subfigure}[b]{0.45\textwidth}
        \centering
        \includegraphics[width=\textwidth]{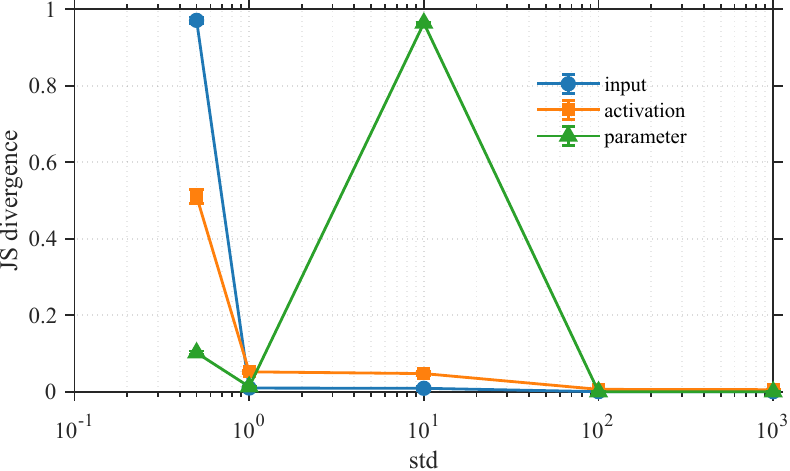}
        \caption{JS divergence between output distributions from consecutive noise magnitudes for ResNet-18 models.}
        \label{fig:appendix_in_act}
    \end{subfigure}
    
    \vspace{0.4cm}
    
    % ------------------- CIFAR-100: input -------------------
    \begin{subfigure}[b]{0.45\textwidth}
        \centering
        \includegraphics[width=\textwidth]{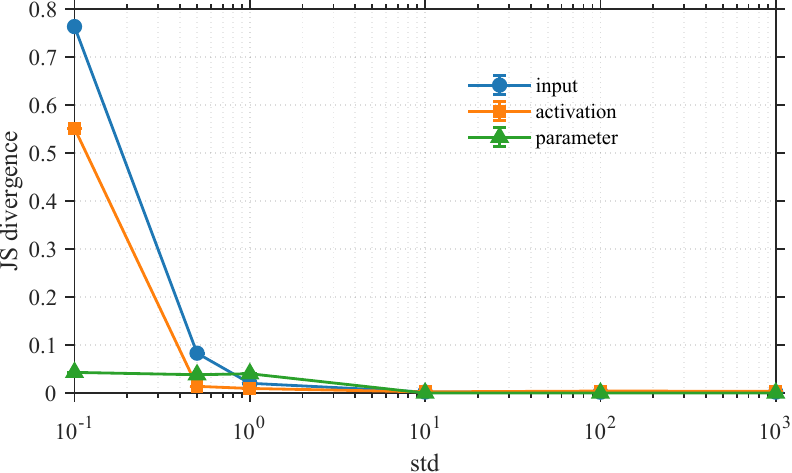}
        \caption{JS divergence between output distributions from different inputs for Inception models.}
        \label{fig:appendix_c100_input}
    \end{subfigure}
    \hfill
    % ------------------- ImageNet: input -------------------
    \begin{subfigure}[b]{0.45\textwidth}
        \centering
        \includegraphics[width=\textwidth]{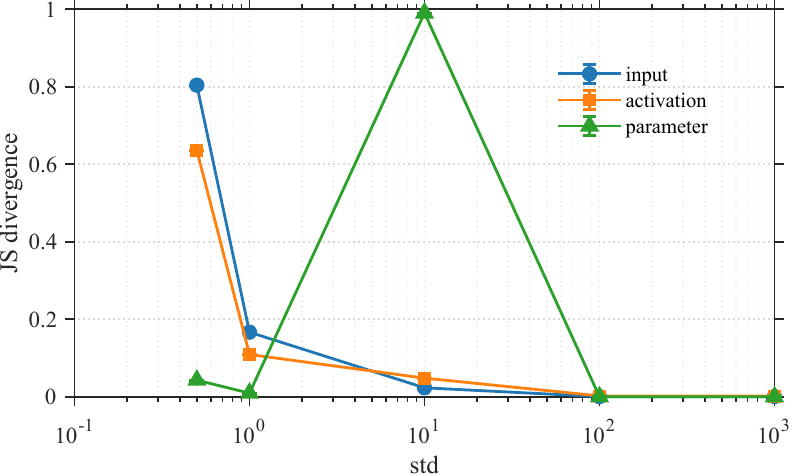}
        \caption{JS divergence between output distributions from consecutive noise magnitudes for Inception models.}
        \label{fig:appendix_in_input}
    \end{subfigure}

    \caption{
    Jensen--Shannon (JS) divergence analysis of output distributions under large perturbations.
(\textbf{Left Column}) Average pairwise JS divergence between output distributions
obtained from different input samples at the same noise magnitude $\sigma$.
(\textbf{Right Column}) JS divergence between the averaged output distributions
at consecutive noise magnitudes.
Experiments are conducted on CIFAR-100 with ten ResNet-18 models (top) and Inception models (bottom), using input-, activation-, and parameter-level perturbations.
As $\sigma$ increases, both metrics rapidly decrease toward zero,
indicating convergence to a sample-independent and stationary output distribution.
While all perturbation types exhibit this convergence,
activation perturbation remains numerically stable and yields smoother transitions,
whereas parameter perturbation becomes unstable under large noise.
    }
    \label{fig:appendix_LPD}
\end{figure*}

\clearpage
\section{Backdoor Attack Configuration}
\label{app:bdsettings}

All backdoor attacks in this paper are implemented using the BackdoorBox toolbox
\cite{backdoorbox}. We adopt the \emph{default configurations} provided by the toolbox for each
attack, including trigger pattern, injection strategy, poisoning rate, and target-class assignment.  
No additional tuning or adjustments are applied. Visualization examples of backdoor triggers are shown below:

\includegraphics[width=\linewidth]{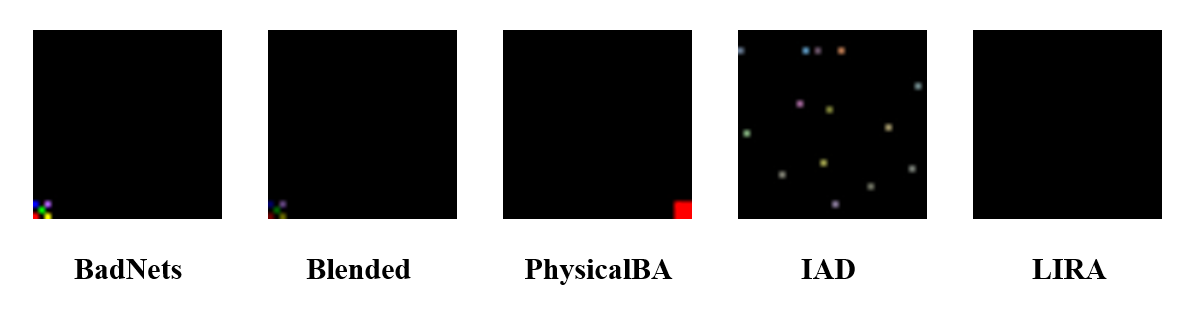}
\captionsetup{width=0.95\linewidth}
\captionof{figure}{An example of backdoor triggers from five different attacks}
\label{fig:trigger}

\section{Transformers: Attenuated Global Collapse}
\label{app:transformer_extension}

Vision transformers are often reported to be more robust to out-of-distribution features,
and backdoor triggers are known to behave similarly to OOD signals.
If this robustness extends to the global representation structure,
backdoored transformers may exhibit a weaker stationary output distribution collapse.
As shown in Figure~\ref{fig:transformer},
backdoored DeiT models initially concentrate on the target class as noise increases,
but partially re-disperse under larger perturbations.
This behavior suggests that, in transformers, poisoned features distort local regions
without fully dominating the global representation space,
highlighting a limitation of collapse-based signals in this architecture.

\begin{figure}[h!]
    \centering
    \includegraphics[width=0.6\linewidth]{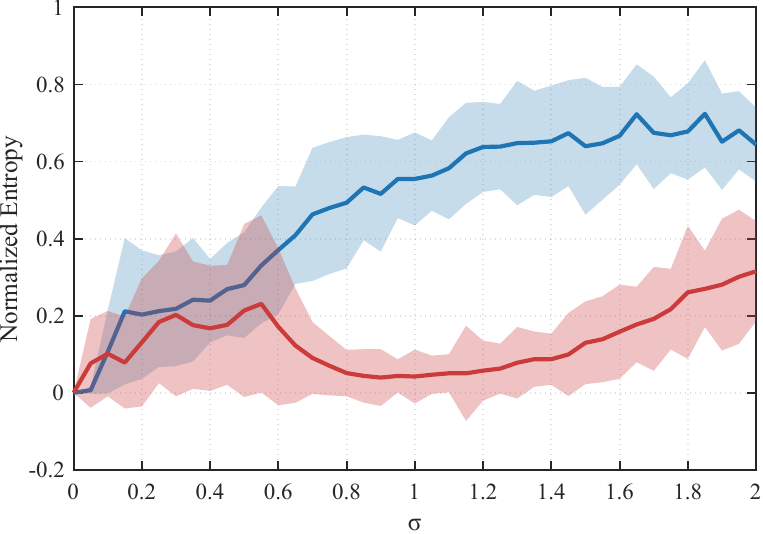}
    \caption{Average normalized entropy of benign and backdoored DeiT-small models under different noise standard deviations.}
    \label{fig:transformer}
\end{figure}
\clearpage
\section{Escape Noise Is Decoupled from Training Confidence}
\label{app:confidence}
This appendix provides a detailed comparison between the evolution of escape noise 
and average softmax confidence over training epochs.
While confidence increases monotonically as training progresses, 
escape noise remains largely stable, indicating that the two quantities capture distinct aspects of model behavior.

\begin{figure}[htb]
\centering
\includegraphics[width=0.6\columnwidth]{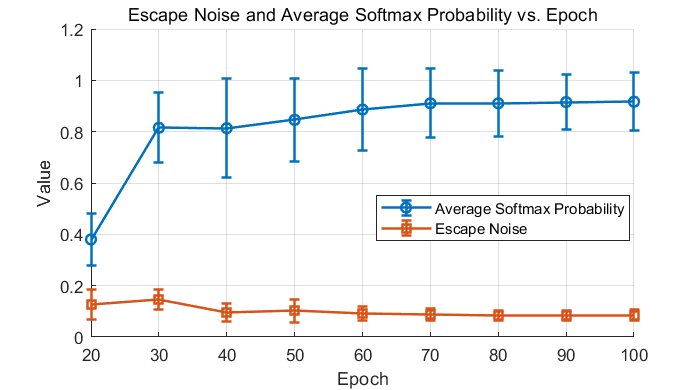}
\caption{Average escape noise and average softmax probability vs. Epoch for 5,000 images sampled from CIFAR-10, with ResNet-18 model.}
\label{fig:ENandConfidence}
\end{figure}

\section{Choice of Noise Type}
\label{app:noisetype}

While activation perturbation primarily uses Gaussian noise, it is not limited to this choice. We compared the KL divergence between output distributions under Gaussian, Laplace, and uniform noise. As shown in Figure~\ref{fig:noisetype}, differences are minimal at high noise levels, suggesting that detection performance is largely determined by noise strength rather than its specific type.

\begin{figure}[h!]
    \centering
    \includegraphics[width=0.6\linewidth]{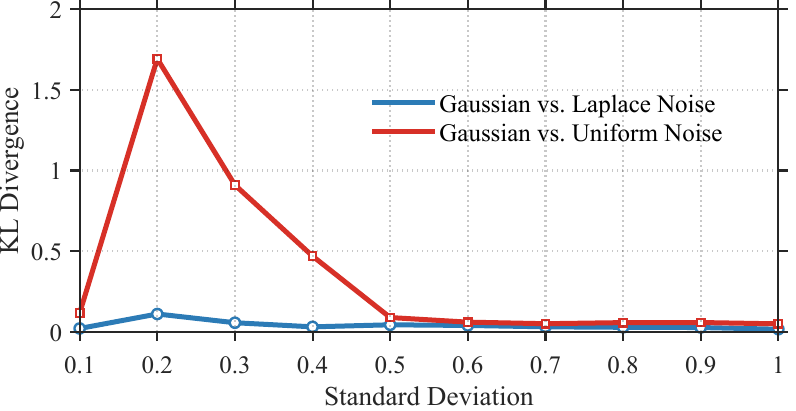}
    \caption{Average KL divergence between output distributions under Gaussian, Laplace, and uniform noise, averaged across ten ResNet-18 CIFAR-10 models.}
    \label{fig:noisetype}
\end{figure}

\clearpage
\section{Single-Layer Activation Perturbation} 
\label{app:singlayer}

Figure~\ref{fig:layerwise} shows that perturbing deeper layers results in higher output dispersion, indicating greater sensitivity in deeper layers. In backdoored models, the output is still concentrated on the target class throughout all layers.

\begin{figure}[h!]
\centering
\includegraphics[width=0.6\linewidth]{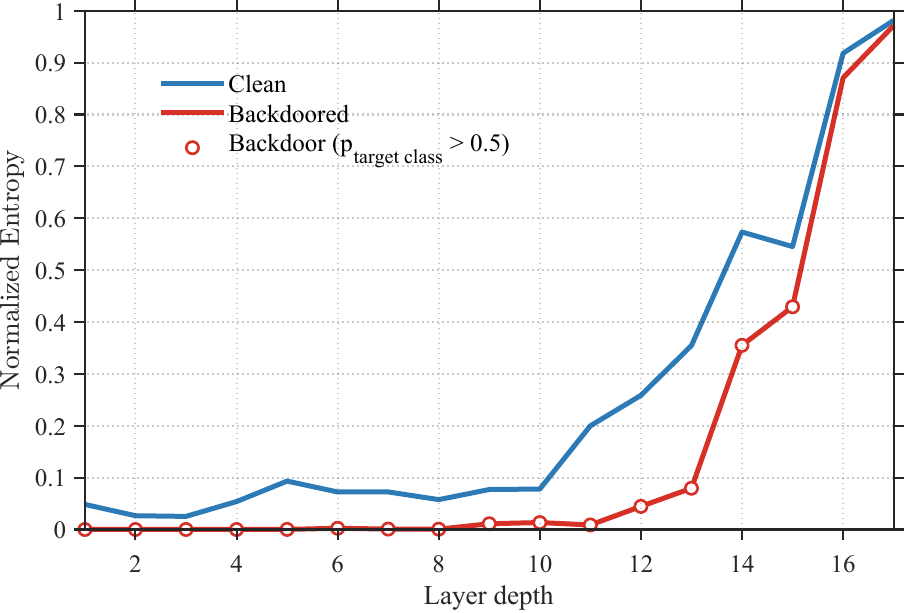}
\caption{
Average normalized entropy when injecting noise into a single ReLU layer for ten ResNet-18 models, CIFAR-100.
}
\label{fig:layerwise}
\end{figure}

\section{Activation Function Choice}
\label{app:activationchoice}

To examine whether the stationary target-aligned behavior depends on the specific choice of activation function,
we repeat the backdoor experiments using models trained with alternative nonlinearities on CIFAR-10.
In addition to ReLU, we consider Leaky ReLU and GELU, which differ substantially in smoothness and negative-region behavior.

For each model, we apply activation perturbation at inference time and measure the average probability
assigned to the backdoor target class across Monte Carlo runs under increasing noise strength.
Figure~\ref{fig:activation_function_ablation} reports the resulting target-class probabilities as a function of the noise standard deviation.

\begin{figure}[h]
    \centering
    \includegraphics[width=0.6\linewidth]{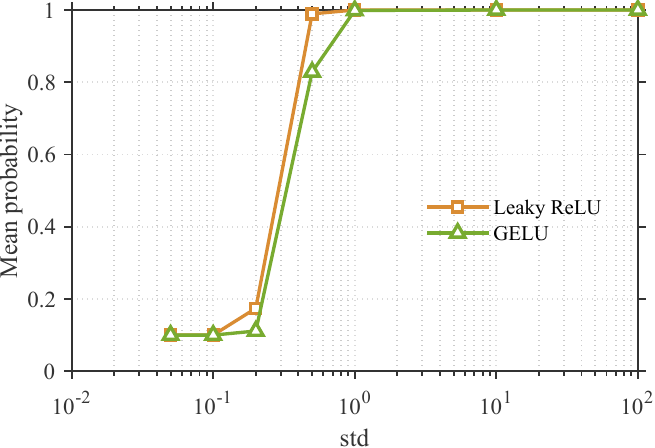}
    \caption{
    Average predicted probability of the backdoor target class under increasing activation noise
    for models trained with different activation functions. Both Leaky ReLU and GELU exhibit a sharp transition toward near-deterministic prediction of the target class,
    similar to ReLU-based models. This indicates that the observed stationary collapse is not specific to ReLU, but reflects a general property of activation perturbation.
    }
    \label{fig:activation_function_ablation}
\end{figure}

\end{document}